\documentclass{article}
\usepackage[utf8]{inputenc}
\usepackage[margin=1.0in]{geometry}
\usepackage{soul,color}
\usepackage{multirow, hhline}
\usepackage{amsfonts, amsmath, amssymb}
\usepackage{tikz}
\usepackage[nobreak=false]{mdframed}
\usepackage{pgf}
\usepackage{enumitem}
\usepackage{mathtools}
\usepackage{bbm}
\usepackage{etoolbox}
\usepackage{graphicx}
\usepackage{listings}
\usepackage{natbib}
\usepackage{algorithm,algorithmic}
\usepackage{hyperref}       
\usepackage[symbol]{footmisc}
\newcommand{\E}{\mathbb{E}}
\newcommand{\norm}[1]{\left\lVert#1\right\rVert}
\newcommand{\wass}[2]{W_1(#1\,\|\, #2)}
\newcommand{\diam}{\mathrm{diam}}
\newenvironment{sproof}{%
	\proof}{\endproof}

\usepackage{amsthm}
\theoremstyle{plain}
\newtheorem{theorem}{Theorem}

\newtheorem{corollary}{Corollary}
\newtheorem{lemma}{Lemma}
\newtheorem{fact}{Fact}

\newtheorem{assumption}{Assumption}

\title{A Contraction Approach to Model-based Reinforcement Learning}
\author{
	Ting-Han Fan and Peter J. Ramadge\\
	Department of Electrical Engineering, Princeton University \\
	\texttt{\{tinghanf,ramadge\}@princeton.edu}
}
\date{}
\begin{document}

\maketitle
\begin{abstract}
	Despite its experimental success, Model-based Reinforcement Learning still lacks a complete theoretical understanding. To this end, we analyze the error in the cumulative reward using a contraction approach. We consider both stochastic and deterministic state transitions for continuous (non-discrete) state and action spaces. This approach doesn't require strong assumptions and can recover the typical quadratic error to the horizon. We prove that branched rollouts can reduce this error and are essential for deterministic transitions to have a Bellman contraction. Our analysis of policy mismatch error also applies to Imitation Learning. In this case, we show that GAN-type learning has an advantage over Behavioral Cloning when its discriminator is well-trained.
\end{abstract}

\section{Introduction}
Reinforcement learning (RL) has attracted much attention recently due to its ability to learn good policies for sequential systems. However, most RL algorithms have a high sample complexity of environment queries (typically in the order of millions). This sample complexity hinders the deployment of RL in practical systems. An intuitive potential solution is to learn an accurate model of the environment's outcome, hence reducing the demand for environment queries. This approach leads to a dichotomy of RL algorithms: training without an environment model is called model-free RL, and training with an environment model is called model-based RL. Model-free RL is often faulted for low exploitation of environment queries, while the performance of model-based RL suffers under model inaccuracy. 

Model-based Reinforcement Learning (MBRL) is nontrivial since RL's sequential nature allows errors to propagate to future time-steps. This fact leads to the planning horizon dilemma \citep{wang2019benchmarking}; a long horizon incurs a large cumulative error, while a short horizon results in shortsighted decisions. We need to understand this trade-off better as it is currently one of the fundamental limitations of model-based RL. 

Most prior error analyses impose a strong assumption in their proofs; e.g., Lipschitz value function \citep{luo2018slbo,xiao2019learning,yu2020mopo} or maximum model error \citep{Janner2019mbpo}. In general, the value function is unlikely to be Lipschitz because its gradient w.r.t. the state can be very large. This event happens when a state perturbation is applied at the stability-instability boundary of a control system, resulting in a large change in value (performance) from a small change in state. For instance, if one perturbs a robot's leg, it may fall and, as a result, receive many negative future rewards.

To mitigate the cumulative reward error, \citet{Janner2019mbpo}  experimentally shows that branched rollouts (short model rollouts initialized by previous real rollouts) help reduce this error and improve experimental results. However, the effectiveness of branched rollouts remains unclear since the experiments of \citet{Janner2019mbpo} use deterministic transitions (MuJoCo \citep{Todorov2012mujoco}). However, their error analysis only applies to stochastic transitions and contains unclear reasoning, see \S\ref{sec:related}. Ideally, we need an analysis framework that applies to both stochastic and deterministic transitions.

Our main contribution is a contraction-based approach to analyze the error of MBRL that applies to both stochastic and deterministic transitions without strong assumptions. Prior work typically makes strong assumptions such as a Lipschitz value function or a maximum model error. To avoid these assumptions and maintain generality, we: (a) provide an analysis framework that applies to both (absolutely continuous) stochastic and deterministic transitions, (b) mostly uses constants in expectation, and (c) does not require a Lipschitz assumption on value functions. Our results also contribute to theoretical explanations of some techniques in deep RL. We prove that branched rollouts significantly reduce the cumulative reward error for both stochastic and deterministic transitions. In particular, branched rollouts are vital for deterministic transitions to have a Bellman contraction. Although prior work also claimed to have a similar conclusion, the analysis is unclear (see \S\ref{sec:related}) and, in any case, does not apply to the deterministic environment in their experiment. Our approach also helps analyze Imitation Learning. We show a GAN-type learning method like Generative Adversarial Imitation Learning \citep{Lee2016gail} is potentially preferable to a supervised learning method like Behavioral Cloning \citep{ross2011reduc, Syed2010} when the discriminator is well-trained.

The primary intuition of our analysis comes from the MBRL problem's asymmetry: the policy mismatch and model mismatch errors, or the terms on the RHS of Eq.~\eqref{eq:decomp}, are the errors of interest and are symmetric when interchanging transitions or policies. However, the objects that control the errors of interest (can be directly made small in training) are asymmetric. At some point, we have to bridge from symmetry to asymmetry. We show that the Bellman flow operator is the key to this bridge. If the Bellman flow operator is a contraction w.r.t. a metric, we can analyze the error of MBRL under that metric regardless of the asymmetry. When we do not have a Bellman contraction, we provide another way inspired by \citet{Syed2010} to analyze the problem and identify the impact of asymmetry. The resulting insight suggests the potential usefulness of the Ensemble Method \citep{kurutach2018metrpo}.

Prior work has done extensive experiments on branched rollouts \citep{Janner2019mbpo}, Generative Adversarial Imitation Learning \citep{Lee2016gail} and the Ensemble Method \citep{kurutach2018metrpo}. Since the empirical evidence in the literature is clear, this work does not include additional experiments.
Instead, we focus on providing an improved theoretical understanding of existing empirical results.

\section{Preliminaries}
Consider an infinite-horizon Markov Decision Process (MDP) represented by 
$\langle\mathcal{S},~\mathcal{A},~ T,~r,~\gamma\rangle.$ 
Here 
$\mathcal{S},~\mathcal{A}$ are finite-dimensional continuous state and action spaces, 
$T(s'|s,a)$ is the transition density of $s'$ 
given $(s,a),$
$r(s,a)$ is the reward function, and 
$\gamma\in(0,1)$ is the discount factor. 
We use $\overline{T}$ to denote a deterministic transition with the density $T(s'|s,a)=\delta(s'-\overline{T}(s,a))$.

Given an initial state distribution $\rho_0$, the goal of reinforcement learning is to learn a (stochastic) policy $\pi$ that maximizes the $\gamma$-discounted cumulative reward $R_\gamma(\rho_0,\pi,T)$, or equivalently, the expected cumulative reward, denoted as $R(\rho_{T,\gamma}^{\rho_0,\pi})$, under the \emph{normalized} occupancy measure $\rho_{T,\gamma}^{\rho_0,\pi}$.
\begin{equation}
\begin{split}
R_\gamma(\rho_0,\pi,T)
& =\mathbb{E}\left[\sum_{i=0}^\infty  \gamma^i r(s_i,a_i)\Big\lvert \rho_0, \pi,T\right]
={ \frac{1}{1-\gamma}}\E_{(s,a)\sim\rho_{T,\gamma}^{\rho_0,\pi}}[r(s,a)]
=R(\rho_{T,\gamma}^{\rho_0,\pi}).\\
\rho_{T,\gamma}^{\rho_0,\pi}(s,a) 
&=  (1-\gamma)\sum_{i=0}^\infty \gamma^i f_i(s,a|\rho_0,\pi,T),
\end{split}
\label{eq:cumulative_reward}
\end{equation}
where $f_i(s,a|\rho_0,\pi,T)$ is the density of $(s,a)$ at step $i$ under $(\rho_0,\pi,T)$. 
Because $(\rho_0,\pi,T,\gamma)$ uniquely determines the occupancy measure, we use $R(\rho_{T,\gamma}^{\rho_0,\pi})$ as an alternative expression for $R_\gamma(\rho_0,\pi,T)$. When $\rho_0,~\gamma$ are fixed, we simplify the notation to $R(\pi,T)$ and $\rho_T^\pi$.

\subsection{Bellman Flow Operator}
In Eq.~\eqref{eq:cumulative_reward}, because each $f_i(s,a|\rho_0,\pi,T)$ uses the same policy, $f_i(s,a)$ and $\rho_{T,\gamma}^{\rho_0,\pi}(s,a)$ can be factored as $f_i(s)\pi(a|s)$ and 
$\rho_{T,\gamma}^{\rho_0,\pi}(s)\pi(a|s)$. 
This allows us to mainly focus on the state distributions. In particular, we define the normalized state occupancy measure $\rho_{T,\gamma}^{\rho_0,\pi}(s)$ as the marginal of $\rho_{T,\gamma}^{\rho_0,\pi}(s,a)$ and show 
(Fact 4, Appendix) 
that
it satisfies a fixed-point equation characterized by a Bellman flow operator $B_{\pi,T}(\cdot)$.
\begin{align}
\begin{split}
\rho_{T,\gamma}^{\rho_0,\pi}(s) 
& \triangleq (1-\gamma)\sum_{i=0}^\infty \gamma^i f_i(s|\rho_0,\pi,T)=B_{\pi,T}(\rho_{T,\gamma}^{\rho_0,\pi}(s)),
\end{split}
\label{eq:state-occu}
\end{align}
where $B_{\pi,T}(\cdot)$ under $(\rho_0,\pi,T)$
and $\gamma$ is defined as:
\begin{align}
\begin{split}
\small B_{\pi,T}(\rho(s))
\triangleq &(1-\gamma)\rho_0(s)+ \gamma \int T(s|s',a')\pi(a'|s')\rho(s')ds'da'
\end{split}
\label{eq:bellman}
\end{align}
$B_{\pi,T}(\cdot)$ is a $\gamma$-contraction w.r.t. total variation distance (see Appendix). Hence, $B_{\pi,T}(\cdot)$ has a unique fixed point, and by Eq.~\eqref{eq:state-occu}, this point is $\rho_{T,\gamma}^{\rho_0,\pi}(s)$. 
This result foreshadows the utility of the Bellman flow operator for analyzing state occupancy measures. Indeed, 
Lemma~\ref{lemma:sym-bridge} exploits the Bellman flow operator to upper bound the distance between state distributions. This is useful for analyzing MBRL. In passing, we note that previous work \citep{Syed2008LP} has made distinct use of a Bellman flow constraint.

\subsection{Model-based RL}
\label{sec:mbrl}
We study the
model-based RL procedure shown in Algorithm \ref{alg:mbrl}, and its variants (e.g., branched rollouts).
\begin{algorithm}[!ht]
	\caption{Model-based RL Algorithm
		\label{alg:mbrl} }
	\begin{algorithmic}[1]
		\REQUIRE Dataset $\mathcal{D}=\emptyset$, policy $\pi_0$, learned transition $\hat{T}$.
		\FOR{$i=1,2,...$}
		\STATE Sample $\mathcal{D}_{i-1}=\{s_t,a_t,s_t'\}$ from real transition $T$ and policy $\pi_{i-1}$.
		\STATE $\mathcal{D}\leftarrow \text{Truncate}(\mathcal{D}\cup \mathcal{D}_{i-1})$
		\STATE Fit $\hat{T}$ using samples in $\mathcal{D}$.
		\STATE $\pi_i=\underset{\pi\in B_{\pi_D}}{\arg\max}~R(\pi,\hat{T})$ 
		\ENDFOR
	\end{algorithmic}
\end{algorithm}
Line 3 deals with the storage of a dataset $\mathcal{D}$ of real transitions. Observe that 
$\mathcal{D}_{i-1}$ is generated by 
$(\rho_0,\pi_{i-1},T)$ and that $\mathcal{D}$ aggregrates the $\mathcal{D}_{i-1}$'s. The policy that generates $\mathcal{D}$, which we call the sampling policy $\pi_D$, is a mixture of previous policies. If $\mathcal{D}_{i-1}$'s have equal sizes, $\pi_D(a|s) = \sum_{j=i-q}^{i-1} \pi_j(a|s)\rho_T^{\pi_j}(s) / \sum_{j=i-q}^{i-1} \rho_T^{\pi_j}(s)$ with $q$ being the truncation level. The larger $q$ is, the more dependent the sampling policy is on previous policies. To facilitate subsequent supervised learning, we need $\mathcal{D}$ to be sufficiently large. However, the larger dataset implies the stronger dependence of the sampling policy on previous policies. For technical reasons (see the final paragraph of the section), we need the current policy $\pi$ and the sampling policy $\pi_D$ to be sufficiently close. Hence we expect a tight truncation, e.g.,~$\mathcal{D}=\mathcal{D}_{i-1}\cup\mathcal{D}_{i-2}$.

Line 4 is a supervised learning task. The objective function is usually the log-likelihood for stochastic transitions or the $\ell_2$ error for deterministic transitions. For stochastic transitions, maximizing likelihood is equivalent to minimizing KL divergence. Hence, by Pinsker's Inequality, the total variation distance 
$\epsilon_{T,\hat{T}}^{\pi_D}=\E_{(s,a)\sim\rho_{T}^{\pi_D} }D_{TV}(T(\cdot|s,a)\,\|\,\hat{T}(\cdot|s,a))$ 
is small. For deterministic transitions, the objective is to minimize $\epsilon_{\ell_2}=\E_{(s,a)\sim\rho_{\overline{T}}^{\pi_D} }
\|\overline{T}(s,a)-\hat{\overline{T}}(s,a)\|_2$.

Line 5 is to maximize model-based cumulative reward $R(\pi,\hat{T})$ under the learned transition $\hat{T}$. Still, the overall goal is to maximize the true cumulative reward $R(\pi,T).$ Note that
\begin{equation}
R(\pi_i, T)-R(\pi_{i-1}, T) =
\underbrace{R(\pi_i, \hat T)-R(\pi_{i-1}, \hat T)}_{\text{m.-b.~policy~improvement}}
+\underbrace{R(\pi_i, T)-R(\pi_i, \hat T) + R(\pi_{i-1},\hat T)- R(\pi_{i-1}, T)}_{\text{reward~errors}}.
\label{eq:reward-diff}
\end{equation}
Hence Line 5 makes an improvement on $R(\pi,T)$ (Eq.~\eqref{eq:reward-diff}$>0$) if the error in cumulative reward $|R(\pi,T)-R(\pi,\hat{T})|$ is small and the model-based policy improvement $R(\pi_i,\hat{T})-R(\pi_{i-1},\hat{T})$ is large. However, the model-based policy improvement is often theoretically intractable. This is because the policy optimization is usually conducted by deep RL algorithms \citep{Fujimoto2018td3,Haarnoja2018sac} but state-of-the-art provable RL algorithms are still limited to linear function approximation \citep{jin20linear,duan2020offpolicy}. Therefore, we assume the model-based policy improvement is sufficiently large and focus on the error in the cumulative reward.

The desired closeness between $\pi$ and $\pi_D$ is achieved by Line 3's truncation and Line 5's constraint to a local ball $B_{\pi_D}$ of $\pi_D.$ Such closeness of policies is commonly used in the literature \citep{luo2018slbo, Janner2019mbpo,yu2020mopo}. Indeed, since $\hat{T}$ is fitted under $\pi_D$ ($\mathcal{D}$'s distribution is $\rho_T^{\pi_D}$), if $\pi$ and $\pi_D$ are far apart, we cannot expect $\hat{T}$ behave like $T$ under $\pi$. Practically, this is not a strong assumption, because we can algorithmically enforce closeness between $\pi$ and $\pi_D$ by constraining the KL divergence between $\pi$ and $\pi_D$. Since it is much easier to control the policy error, the model error is the dominating error in MBRL. Hence we focus on the dependency of the cumulative model error w.r.t. the horizon.

\section{Related Work}\label{sec:related}

There have been many experimental studies of model-based RL. Evidence in \citet{Gu2016mbdq} and \citet{Nagabandi2018mbmf} suggests that for continuous control tasks, vanilla MBRL \citep{Sutton1991dyna} hardly surpasses model-free RL, unless using a linear transition model or a hybrid model-based and model-free algorithm. To enhance the applicability of MBRL, the Ensemble Method is widely adopted in the literature, since it helps alleviate overfitting in a neural network (NN) model. Instances of this approach include an ensemble of deterministic NN transition models \citep{kurutach2018metrpo}, an ensemble of probabilistic NN transition models \citep{Kurtland2018pets} with model predictive control \citep{mpc} or ensembles of deterministic NN for means and variances of rollouts with different horizons \citep{Jacob2018steve}. In addition to training multiple models, \citet{clavera2018meta} leverages meta-learning to train a policy that can quickly adapt to new transition models. \citet{wang2019benchmarking} provides useful benchmarks of various model-based RL methods.

On the theoretical side, for stochastic state transitions the error in the cumulative reward is quadratic in the length of model rollouts. Specifically, \citet[Theorem A.1]{Janner2019mbpo} provides the bound
\vspace{-1mm}
\begin{equation}
R(\pi,T)-R(\pi,\hat{T})\geq -\frac{2\gamma r^{\max}}{(1-\gamma)^2}(\epsilon_m+2\epsilon_{\pi})-\frac{4\epsilon_\pi r^{\max}}{1-\gamma},
\label{eq:mbpo-err}
\end{equation}
where $\epsilon_m=\max_t \E_{s,a\sim \rho_{\pi_D,t}}D_{TV}(T(\cdot|s,a)\,\|\, \hat{T}(\cdot|s,a))$,\vspace{1mm}
$\epsilon_{\pi}=\max_sD_{TV}(\pi_D(\cdot|s)\,\|\, \pi(\cdot|s))$ and $\rho_{\pi_D,t}$ is the density of $(s,a)$ at step 
$t$ following $(\rho_0,\pi_D,T)$. 
For deterministic state transitions and an  $L$-Lipschitz value function $V(s)$, \citet[Proposition 4.2]{luo2018slbo} shows

\begin{equation}\label{eq:slbo-err}
\left | R(\pi,T)-R(\pi,\hat{T}) \right | 
\leq
\frac{\gamma}{1-\gamma} L ~\E_{\substack{s\sim \rho_T^{\pi_D}\\a\sim\pi(\cdot|s)}}
\| {\bar{T}(s,a)-\hat{\bar{T}}(s,a)} \| 
+2\frac{\gamma^2}{(1-\gamma)^2} \delta~\diam_{\mathcal{S}},
\end{equation}
where $\delta=\E_{s\sim \rho_T^{\pi}}\sqrt{D_{KL}
(\pi(\cdot|s)\,\|\,\pi_D(\cdot|s))}$ and $\diam_S$ is the diameter of $\mathcal{S}$. 

In practice, we enforce $\pi_D$ and $\pi$ to be close, so the model error terms dominate in Eq.~(\ref{eq:mbpo-err}) and (\ref{eq:slbo-err}). Eq.~(\ref{eq:slbo-err}) looks sharper since the model error is correlated with a linear rather than quadratic term of the expected rollout length $(1-\gamma)^{-1}.$ However, since the value function represents the cumulative reward, it's Lipschitz constant (assuming it exists) 
can be $O((1-\gamma)^{-1}).$ So it is hard to compare Eq.~(\ref{eq:mbpo-err}) and (\ref{eq:slbo-err}). While a Lipschitz value function is commonly assumed in the literature \citep{luo2018slbo,xiao2019learning,yu2020mopo}, this is hard to verify in practice, and, even if it holds, 
the constant can be very large. To avoid strong assumptions, we do not assume a Lipschitz value function. 
In addition, we enhance the results of \citet{Janner2019mbpo}, 
by showing their constants ``in maxima'' can be replaced by constants ``in expectation''.

A major contribution of \citet{Janner2019mbpo} 
is the use of branched rollouts generated by 
$(\rho_T^{\pi_D},\pi,\hat{T})$. By Theorem 4.3 in \citep{Janner2019mbpo}, branched rollouts of length $k$ satisfy
\begin{align}
R(\pi,T) - R^{\text{branch}}(\pi)\geq -2r^{\max}\Big[ \frac{\gamma^{k+1}\epsilon_{\pi}}{(1-\gamma)^2}
\!+\!\frac{\gamma^k\epsilon_{\pi}}{1-\gamma}
\!+\!\frac{k\epsilon_m}{1-\gamma} \Big],
\label{eq:mbpo-br-err}
\end{align}
with the same constants as Eq.~(\ref{eq:mbpo-err}). Eq.~(\ref{eq:mbpo-br-err}) implies that if the model is almost perfect ($\epsilon_m\approx0$), the error is dominated by the policy error $\gamma^k\epsilon_\pi.$
Since $\gamma<1,$ the minimal error is attained at large $k$. This suggests that a near perfect model helps correct the error due to the mismatch of the sampling policy and current policy. Still, with an almost perfect model, the problem is reduced to off-policy RL, which always suffers from policy mismatch error \citep{duan2020offpolicy}. Attaining the minimal error at large branch length $k$ 
also contradicts the fact that the error accumulates over the trajectory \citep{wang2019benchmarking,xiao2019learning}. The error propagation in the MBRL system implies the compounding error always increases with length. If we accept Eq.~\eqref{eq:mbpo-br-err}, there is still an important gap since Eq.~\eqref{eq:mbpo-br-err} is for stochastic transitions, but the experiments in \citet{Janner2019mbpo} used deterministic transitions. Our analysis shows the error of branched rollouts for both stochastic and deterministic transitions increases in the expected branched length $(1-\beta)^{-1}.$ So we always favor short lengths and are free from the issues mentioned above.

\section{Main Result}
\label{sec:main}
As discussed in \S\ref{sec:mbrl}, we focus on the error in the cumulative reward $|R(\pi,T)-R(\pi,\hat{T})|$ in MBRL settings. To do so we use the triangle inequality:
\begin{equation}
|R(\pi,T)-R(\pi,\hat{T})|
\leq \underbrace{|R(\pi,T)\!-\!R(\pi_D,T)|}_{\substack{\text{controlled~by}~\epsilon_{\pi_D,\pi}^{T}}}+\underbrace{|R(\pi_D,T)\!-\!R(\pi_D,\hat{T})|}_{\substack{\text{controlled~by}~\epsilon_{T,\hat{T}}^{\pi_D}~\text{or}~\epsilon_{\ell_2}}} +\underbrace{|R(\pi_D,\hat{T})\!-\!R(\pi,\hat{T})|}_{\substack{\text{controlled~by}~\epsilon_{\pi_D,\pi}^{\hat{T}}}}.
\label{eq:decomp}
\end{equation}
The error terms on the RHS of Eq.~\eqref{eq:decomp} result from policy mismatch ($1^{\text{st}}$ and 3$^{\text{rd}}$ terms) and transition mismatch ($2^\text{nd}$ term). Moreover, since these errors are controlled by the discrepancies between $T, \hat T$ and between $\pi, \pi_D$, the errors can be made small in the MBRL training (see the discussion in \S\ref{sec:mbrl}).

The discrepancy between policies $\pi_D$ and $\pi$ is measured by the total variation (TV) distance:
\begin{align*}
\epsilon_{\pi_D,\pi}^{T}=\E_{s\sim\rho_{T}^{\pi_D} }D_{TV}\left ( \pi_D(\cdot|s) 
~ \| ~ \pi(\cdot|s)\right ) 
\quad 
\textrm{and} \quad \epsilon_{\pi_D,\pi}^{\hat{T}}
=\E_{s\sim\rho_{\hat{T}}^{\pi} }D_{TV}(\pi_D(\cdot|s)~ \| ~\pi(\cdot|s)).
\end{align*}
The expectation in $\epsilon_{\pi_D,\pi}^{T}$ is over $\rho_{T}^{\pi_D}$ and in $\epsilon_{\pi_D,\pi}^{\hat{T}}$ over $\rho_{\hat{T}}^{\pi}$. This allows us to measure policy discrepancy under the real dataset $\mathcal{D}$ (distributed as $\rho_{T}^{\pi_D}$) and the simulated environment (distributed as $\rho_{\hat{T}}^{\pi}$).

The discrepancies between real and learned transitions $T,~\hat{T}$ are measured by (a) TV distance for stochastic transitions and (b) $\ell_2$ error for deterministic ones. 
\begin{align*}
\textrm{(a)}~\epsilon_{T,\hat{T}}^{\pi_D} 
=\E_{(s,a)\sim\rho_{T}^{\pi_D} }D_{TV}
\left (T(\cdot|s,a)~ \| ~\hat{T}(\cdot|s,a)\right )
~\textrm{and}~
\textrm{(b)}~\epsilon_{\ell_2} 
=\E_{(s,a)\sim  \rho_{\overline{T}}^{\pi_D}  } 
\|\overline{T}(s,a)-\hat{\overline{T}}(s,a)\|_2.
\end{align*}

From the RHS of Eq.~(\ref{eq:decomp}), the policy mismatch errors (1$^{\text{st}}$ and 3$^{\text{rd}}$ terms) are invariant under exchange of $\pi$ and $\pi_D.$ We hence call these terms
``symmetric in $(\pi,\pi_D)$". 
Similarly, the transition mismatch error (2$^\text{nd}$ term) is symmetric in $(T,\hat{T})$. However, the terms that control them, $\epsilon_{\pi_D,\pi}^{T},$ $\epsilon_{\pi_D,\pi}^{\hat{T}}$, and $\epsilon_{T,\hat{T}}^{\pi_D},$ $\epsilon_{\ell_2}$ 
are asymmetric; the first two in $(\pi,\pi_D)$ and the last in $(T,\hat T).$ To bridge these symmetric and asymmetric quantities, we establish the following:

\begin{equation}
|R(\rho_1)\!-\!R(\rho_2)|\underset{(*)}{\leq} C\!\times\!\{D_{TV}(\rho_1\lvert\rvert\rho_2)~\text{or}~\wass{\rho_1}{\rho_2}\}
\underset{(**)}{\leq} C'\!\times\! \{\epsilon_{\pi_D,\pi}^{T}, \epsilon_{\pi_D,\pi}^{\hat{T}},\epsilon_{T,\hat{T}}^{\pi_D},\text{~or}~\epsilon_{\ell_2} \}.
\label{eq:rough}
\end{equation}

Eq.~\eqref{eq:rough} outlines the proof technique using notations in Eq.~\eqref{eq:cumulative_reward}, with $C,~C'$ being underdetermined constants. Inequality $(*)$ upper bounds the cumulative reward error by one of the symmetric quantities w.r.t.~occupancy measures $(\rho_1(s,a),\rho_2(s,a))$: TV distance for stochastic transitions or 1-Wasserstein distance \citep{Villani2008opt_trans} for deterministic transitions. Inequality $(**)$ upper bounds the symmetric quantities by one of the asymmetric ones, using the contraction of the Bellman flow operator (if it holds). 

While the Bellman flow operator is a contraction w.r.t.~TV distance, this may not hold w.r.t.~$W_1$ distance. We address this situation in \S\ref{sec:main-det-weak}. Although we use $W_1$ distance as an intermediate step to analyze deterministic transitions, we finally upper bound $W_1$ distance by $\ell_2$ error, as outlined by Inequality $(**)$. This avoids the need to minimize $W_1$ error using a 
Wasserstein GAN \citep{arjovsky2017wasserstein} 
or other optimization techniques \citep{Gabriel2019compwass}. 

In the following subsections, we first analyze the policy error, then the transition error when we have: (1) absolutely continuous stochastic transitions, (2) deterministic transitions with strong continuity, and (3) deterministic transitions with weak continuity. 
Cases (1) and (2) have Bellman contractions yielding sharp two-sided bounds. Case (3) uses a bounding technique inspired by \citet[Lemma 2]{Syed2010} to establish a one-sided bound. By combining the policy error with the transition errors, we obtain corresponding MBRL errors. Full proofs are in the Appendix. 

\subsection{Symmetry Bridge Lemma and Policy Mismatch Error}
We start by introducing a key lemma. Then we will use it to analyze the policy mismatch error.
\subsubsection{Symmetry Bridge Lemma}
Lemma~\ref{lemma:sym-bridge} (following \citet[Corollary 2.4]{Conrad2014}) is a key to analyze both policy mismatch and transition mismatch errors through contractions.

\begin{lemma}
	\label{lemma:sym-bridge}
	Let $B$ be a Bellman flow operator with fixed-point $\rho^\star$ and $\rho$ be a state distribution. If $B$ is a $\eta$-contraction w.r.t. some metric $\norm{\cdot}$, then $$\norm{\rho-\rho^\star}\leq \norm{\rho-B(\rho)}/(1-\eta).$$
\end{lemma}
Recall from inequality $(**)$ of Eq.~(\ref{eq:rough}), we need to bridge from symmetric quantities to asymmetric ones. Lemma~\ref{lemma:sym-bridge} constructs a bridge for this purpose. The LHS is symmetric (invariant under exchange) in $(\rho,~\rho^\star)$. Also, the RHS is asymmetric in $(\rho,~\rho^\star)$ because $B$ is associated with $\rho^\star$. Hence, one can upper bound symmetric quantities using asymmetric ones if the contraction is given.

\subsubsection{Policy Mismatch Error}
The policy mismatch error is analyzed by Eq.~(\ref{eq:rough}). Note that the discrepancy between policies is measured by TV distance and that the Bellman flow operator is a contraction w.r.t. TV distance. Lemma~\ref{lemma:sym-bridge} establishes 
the inequality $(**)$ of Eq.~(\ref{eq:rough}). The next Lemma is used to verify inequality $(*)$.
\begin{lemma}
	If $0\leq r(s,a)\leq r^{\max}$, then $$|R(\rho_1)- R(\rho_2)| \leq D_{TV}(\rho_1 \| \rho_2) r^{\max}/(1-\gamma).$$
	\label{lemma:err-tv-occu}
\end{lemma}
\vspace{-6mm}
The following theorem establishes the upper bounds of policy mismatch errors (1$^{\text{st}}$ and 3$^{\text{rd}}$ terms of Eq.~\eqref{eq:decomp}). Lemmas~\ref{lemma:sym-bridge} and \ref{lemma:err-tv-occu} are used in its proof.

\begin{theorem}\label{thm:err-tv-policy}
	If $0\leq r(s,a)\leq r^{\max}$ and  $\epsilon_{\pi_D,\pi}^T=\E_{s\sim \rho_T^{\pi_D}}[D_{TV}(\pi_D(\cdot|s) \,\|\, \pi(\cdot|s))]$, then 
	$$|R(\pi_D,T)-R(\pi,T)|\leq \epsilon_{\pi_D,\pi}^T r^{\max}\Big(\frac{1}{1-\gamma}+\frac{\gamma}{(1-\gamma)^2}\Big).$$
\end{theorem}
\begin{sproof}
	By Lemma~\ref{lemma:err-tv-occu}, it is enough to upper bound $D_{TV}(\rho_T^{\pi_D}(s,a)\lvert\lvert \rho_T^{\pi}(s,a))$: 
	\begin{equation*}
	\begin{split}
	&D_{TV}(\rho_T^{\pi_D}(s,a) \| \rho_T^{\pi}(s,a))\\ 
	\leq &D_{TV}(\rho_T^{\pi_D}(s)\pi_D(a|s) \| \rho_T^{\pi_D}(s)\pi(a|s))+D_{TV}(\rho_T^{\pi_D}(s)\pi(a|s) \| \rho_T^{\pi}(s)\pi(a|s))\\
	\leq &\epsilon_{\pi_D,\pi}^T + \frac{1}{1-\gamma} D_{TV}(B_T^{\pi_D}(\rho_T^{\pi_D}(s)) \| B_T^\pi(\rho_T^{\pi_D}(s)))\\
	\leq &\epsilon_{\pi_D,\pi}^T + \frac{\gamma}{1-\gamma}\epsilon_{\pi_D,\pi}^T,
	\end{split}
	\end{equation*}
	where the second inequality follows from Lemma \ref{lemma:sym-bridge} and the fixed-point property.
\end{sproof}

\subsubsection{Application to Imitation Learning}
We now make the following interesting side observation. Imitation learning \citep{Syed2010,Lee2016gail} 
is matching the demonstrated policy and the generator policy. 
Because Theorem~\ref{thm:err-tv-policy} is about policy mismatch error, it is applicable to imitation learning. Observe that the objective of GAIL is JS (Jensen-Shannon) divergence when its discriminator is well-trained and that Behavior Cloning's objective is KL divergence. We can use Pinsker's Inequality to upper bound these divergences and translate Theorem~\ref{thm:err-tv-policy} and Lemma~\ref{lemma:err-tv-occu}, yielding:

\begin{corollary}[Error of Behavioral Cloning]
	\setlength\topsep{0pt}
	Let $\pi_D$ and $\pi$ be the expert and agent policy. If $0\leq r(s,a)\leq r^{\max}$ and $\E_{s\sim\rho_{T}^{\pi_D}} D_{\textrm{KL}}(\pi_D(\cdot|s) \| \pi(\cdot|s)) \leq \epsilon_{\textrm{BC}}$, then 
	$$
	|R(\pi_D,T)-R(\pi,T)|\leq \sqrt{\epsilon_{\textrm{BC}}/2}~r^{\max}\Big(\frac{1}{1-\gamma}+\frac{\gamma}{(1-\gamma)^2}\Big).$$
\end{corollary}

\begin{corollary}[Error of GAIL]
	\setlength\topsep{0pt}
	Let $\pi_D$ and $\pi$ be the expert and agent policy, respectively. If $0\leq r(s,a)\leq r^{\max}$ and $D_{\textrm{JS}}(\rho_{T}^{\pi_D} \| \rho_{T}^{\pi})\leq \epsilon_{\textrm{GAIL}}$, then 
	$$|R(\pi_D,T)-R(\pi,T)|\leq \sqrt{2\epsilon_{\textrm{GAIL}}}~
	r^{\max}/(1-\gamma).$$
\end{corollary}
Observe that Behavioral Cloning's error is quadratic w.r.t. the expected horizon $(1-\gamma)^{-1}$ while GAIL's is linear. This suggests that when the discriminator is well-trained, GAN-style imitation learning, like GAIL, has an advantage.

\subsection{MBRL with Stochastic Transitions}
If the true transitions are stochastic, we can learn $\hat{T}$ by maximizing the likelihood, or equivalently by minimizing the KL divergence. To ensure the KL divergence is defined on a continuous state space, we assume the transition probability is absolutely continuous w.r.t. the state space, i.e., there is a density function and hence no discrete or singular continuous measures \citep{analysis}. The following theorem then follows by the proof of Theorem~\ref{thm:err-tv-policy}.

\begin{theorem}
	If $0\leq r(s,a)\leq r^{\max}$ and $\epsilon_{T,\hat{T}}^{\pi_D}=\E_{(s,a)\sim \rho_T^{\pi_D}}[D_{TV}(T(\cdot|s,a) \| \hat{T}(\cdot|s,a))]$, then 
	$$|R(\pi_D,T)-R(\pi_D,\hat{T})|\leq \epsilon_{T,\hat{T}}^{\pi_D} ~ 
	r^{\max}\gamma(1-\gamma)^{-2}.  $$
	\label{thm:err-tv-trans}
\end{theorem}
\vspace{-4mm}
Theorems \ref{thm:err-tv-policy} and \ref{thm:err-tv-trans} yield the following result for MBRL with absolutely continuous stochastic transitions.
\begin{corollary}
	Assume $0\leq r(s,a)\leq r^{\max}$ and let $\epsilon_{T,\hat{T}}^{\pi_D}\triangleq \E_{(s,a)\sim\rho_T^{\pi_D}}D_{TV}(T(\cdot|s,a) ~\|~ \hat{T}(\cdot|s,a))$,\\
	
	~~~~~~~~~~~~~~$\epsilon_{\pi_D,\pi}^{T}\triangleq \E_{s\sim\rho_T^{\pi_D}}D_{TV}(\pi_D(\cdot|s)~\|~ \pi(\cdot|s))$ and $\epsilon_{\pi_D,\pi}^{\hat{T}}\triangleq \E_{s\sim\rho_{\hat{T}}^{\pi}}D_{TV}(\pi_D(\cdot|s)~\|~ \pi(\cdot|s)).$ Then
	\begin{align*}
	|R(\pi,T)-R(\pi,\hat{T})|
	\leq (\epsilon_{T,\hat{T}}^{\pi_D}+\epsilon_{\pi_D,\pi}^{T}
	+\epsilon_{\pi_D,\pi}^{\hat{T}})
	\frac{r^{\max}\gamma}{(1-\gamma)^2}
	+(\epsilon_{\pi_D,\pi}^{T}+\epsilon_{\pi_D,\pi}^{\hat{T}})\frac{r^{\max}}{(1-\gamma)}.
	\end{align*}
	\label{cor:err-stoch}
\end{corollary}
\vspace{-4mm}
Comparing the result in Corollary~\ref{cor:err-stoch} with the prior results in Eq.~(\ref{eq:mbpo-err}), we sharpen the bounds by changing the constants from maxima to expected values.

\subsubsection{MBRL with Branched Rollouts}
Corollary~\ref{cor:err-stoch} indicates that the model error term $\epsilon_{T,\hat{T}}^{\pi_D} r^{\max}\gamma/(1-\gamma)^2$ is quadratic w.r.t. the expected rollout length $(1-\gamma)^{-1}$, which makes MBRL undesirable for long rollouts and leads to the planning horizon dilemma. An intuitive countermeasure is to use short rollouts that share similar distributions with the long ones. This leads to the idea of branched rollouts. Throughout the rest of the paper, $\beta>0$ will denote the branched discount factor with $\beta <\gamma.$ We define a {\em branched rollout} with discount factor $\beta$, to be a rollout following the laws of $(\rho_{T,\gamma}^{\pi_D},~\pi,~\hat{T})$. 
Intuitively, these are rollouts initialized on the states of previous real long rollouts, $\rho_{T,\gamma}^{\pi_D}(s)$, and then run a few steps under policy $\pi$ and model $\hat{T}$.

The occupancy measure of branched rollouts is $\rho_{\hat{T},\beta}^{\rho_{T,\gamma}^{\pi_D},\pi}$ where the superscripts $\rho_{T,\gamma}^{\pi_D},\pi$ indicate the initial state distribution and policy, 
and the subscripts $\hat{T},\beta$ indicate the transition and discount factor. Branched rollouts are short by construction, but it is unclear whether their distribution is similar to that of long rollouts. This is verified by the following Lemma.

\begin{lemma}\label{lemma:short_occu_bound}
	Let $\gamma> \beta$ be the discount factors of long and short rollouts, and $\pi_D$ and $T$ be the sampling policy and the real transition. Then 
	$$
	D_{TV}(\rho_{T,\gamma}^{\pi_D} ~\|~ 
	\rho_{T,\beta}^{\rho_{T,\gamma}^{\pi_D},\pi_D}  ) 
	\leq (1-\gamma)\beta/(\gamma-\beta).
	$$
\end{lemma}
By Lemma~\ref{lemma:short_occu_bound}, $\rho_{T,\gamma}^{\pi_D}$ and 
$\rho_{T,\beta}^{\rho_{T,\gamma}^{\pi_D},\pi_D} $ 
are close if $\beta$ is small. Hence once the pairs $(\pi,\pi_D)$ and $(T,\hat{T})$ are close, the distribution of branched rollouts is similar to that of long real rollouts, and the error in cumulative reward is small. This is given in detail below.

\begin{corollary}\label{cor:err-stoch-branch}
	Let $0\leq r(s,a)\leq r^{\max},$~~ $\epsilon_{\pi_D,\pi}^{T,\gamma}
	=\E_{s\sim\rho_{T,\gamma}^{\pi_D}}D_{TV}(\pi_D(\cdot|s) ~\|~ \pi(\cdot|s))$,\\
	
	$\epsilon_{\pi_D,\pi}^{\hat{T},\beta}
	=\E_{s\sim\rho_{\hat{T},\beta}^{\rho_{T,\gamma}^{\pi_D},\pi}}D_{TV}(\pi_D(\cdot|s)~\|~\pi(\cdot|s))$, and
	$\epsilon_{T,\hat{T}}^{\pi_D,\beta} =\E_{(s,a)\sim\rho_{T,\beta}^{\rho_{T,\gamma}^{\pi_D},\pi_D}}D_{TV}(T(\cdot|s,a) ~\|~ \hat{T}(\cdot|s,a)).$ Then
	$$ \Big|R_\gamma(\rho_0,\pi,T)-\frac{1-\beta}{1-\gamma}R_\beta(\rho_{T,\gamma}^{\pi_D},\pi,\hat{T})\Big|\leq r^{\max}\Big(\frac{\epsilon_{\pi_D,\pi}^{T,\gamma}\gamma}{(1-\gamma)^2} + \frac{(\epsilon_{T,\hat{T}}^{\pi_D,\beta} + \epsilon_{\pi_D,\pi}^{\hat{T},\beta})\beta}{(1-\beta)(1-\gamma)} + \frac{\epsilon_{\pi_D,\pi}^{T,\gamma}+\epsilon_{\pi_D,\pi}^{\hat{T},\beta}}{1-\gamma} + \frac{\beta}{\gamma-\beta}\Big).$$
\end{corollary}
\vspace{-4mm}
\begin{sproof}
	Decompose the error as follows and then apply Theorems~\ref{thm:err-tv-policy},~\ref{thm:err-tv-trans}, 
	and
	Lemmas~\ref{lemma:err-tv-occu},~\ref{lemma:short_occu_bound}.
	\begin{align*}
	&| R_\gamma  (\rho_0,\pi,T)-\frac{1-\beta}{1-\gamma}
	R_\beta(\rho_{T,\gamma}^{\pi_D},\pi,\hat{T}) |\\
	\leq&|R_\gamma(\rho_0,\pi,T)-R_\gamma(\rho_0,\pi_D,T)| + | R_\gamma(\rho_0,\pi_D,T)-\frac{1-\beta}{1-\gamma}R_\beta(\rho_{T,\gamma}^{\pi_D},\pi_D,T) |\\
	&+ \frac{1-\beta}{1-\gamma} | R_\beta(\rho_{T,\gamma}^{\pi_D},\pi_D,T)-R_\beta(\rho_{T,\gamma}^{\pi_D},\pi_D,\hat{T}) | +\frac{1-\beta}{1-\gamma} | R_\beta(\rho_{T,\gamma}^{\pi_D},\pi_D,\hat{T})-R_\beta(\rho_{T,\gamma}^{\pi_D},\pi,\hat{T}) |.
	\end{align*}
\end{sproof}
\vspace{-4mm}
Notice $\epsilon_{T,\hat{T}}^{\pi_D,\beta}$ is controlled by supervised learning since it is evaluated on dataset $\mathcal{D}$. Because branched rollouts are shorter than normal rollouts, the branch cumulative reward is rescaled to 
$\frac{1-\beta}{1-\gamma}R_\beta(\rho_{T,\gamma}^{\pi_D},\pi,\hat{T})$ for comparison to normal rollouts. Compared with Corollary~\ref{cor:err-stoch}, the model error term's dependency on the rollout lengths is reduced from $O((1-\gamma)^{-2})$ to $O((1-\gamma)^{-1}(1-\beta)^{-1})$. This shows that branched rollouts greatly reduce the cumulative reward error.

Corollary~\ref{cor:err-stoch-branch} shows the error in cumulative reward increases in $\beta$, or equivalently in the expected branched length $(1-\beta)^{-1}.$ Thus our result is free from the issue of previous work Eq.~\eqref{eq:mbpo-br-err} (Note: $(1-\beta)^{-1}$ corresponds to the branch length $k$ in Eq.~\eqref{eq:mbpo-br-err}). It is tempting to set $\beta=0$ to minimize the reward error. However, if $\beta=0$, each branched rollout is composed of a single point drawn from $\rho_{T,\gamma}^{\pi_D}$. This means that the branched rollouts access neither $T$ nor $\hat{T}$, so we will learn a policy that only optimizes on initial states and has no concern for the future. For example, the reward of MuJoCo environment \citep{Todorov2012mujoco} is typically $r(s,a)=\text{velocity}(s) - \norm{a}_2^2.$ To maximize cumulative reward on branched rollouts with $\beta=0$, the optimal policy will shortsightedly select $a=0$ for any $s.$

The branched rollout makes a trade-off between policy improvement and reward error, as discussed in \S\ref{sec:mbrl}. The policy improvement $R_\beta(\rho_T^{\pi_D},\pi_i,\hat{T})-R_\beta(\rho_T^{\pi_D},\pi_{i-1},\hat{T})$ in branched rollouts benefits from a larger $\beta$, while the reward error, as shown in Corollary~\ref{cor:err-stoch-branch} and ~\ref{cor:err-deter-lip-branch}, favors smaller $\beta$. In MuJoCo, \citet[Appendix C]{Janner2019mbpo} says the branched length is chosen as 2 in early epochs and may stay small or gradually increase to 16 or 26 later. This suggests for continuous-control (MuJoCo) tasks, $\beta \approx0.9$ is enough to balance policy improvement and reward error.

\subsection{MBRL with Deterministic Transitions}
We discuss deterministic transitions under (a) strong, and (b) weak Lipschitz assumptions. The main difference is the validity of Lemma~\ref{lemma:cond-contr}, which is controlled by the smoothness of the deterministic transition.

\subsubsection{Strong Lipschitz Continuity}
A major difficulty in analyzing deterministic transitions is that TV distance is not suitable for 
comparing $\overline{T}$ and $\hat{\overline{T}}$. Indeed, for any fixed $(s,a)$, $D_{TV}(\delta(s'-\overline{T}(s,a))~\|~\delta(s'-\hat{\overline{T}}(s,a)))=1$ once $\overline{T}(s,a)\neq \hat{\overline{T}}(s,a)$. Moreover, the model error is controlled by $\epsilon_{\ell_2}=\E_{(s,a)\sim \rho_{\overline{T}}^{\pi_D}}
\|\overline{T}(s,a)-\hat{\overline{T}}(s,a) \|_2$, but the $\ell_2$ error is not a distance metric for distributions. To control the distance between distributions through an $\ell_2$ error, we can select a distance metric for distributions that is upper bounded by $\ell_2$ error. The 1-Wasserstein distance is a good candidate:
\begin{equation}
\wass{\rho_1(s)}{\rho_2(s)}
=\!\!\underset{J(s_1,s_2)\in \Pi(\rho_1,\rho_2)}{\inf}\E_{J}\norm{s_1\!-\!s_2}_2
\end{equation}
where the infimum is over joint distributions $J(s_1,s_2)$ with marginals $\rho_1(s_1)$, $\rho_2(s_2)$. To apply Eq.~(\ref{eq:rough}), it is crucial to use a metric under which the Bellman flow operator is a contraction. To ensure this holds for $W_1$ distance, we make the following Lipschitz assumptions on the transitions and policies.

\vspace{-2mm}
\subsubsection*{Assumption 1}
\begin{itemize}
	\setlength\topsep{2pt}
	\setlength\itemsep{2pt} 
	\setlength\parsep{2pt}
	\setlength{\itemindent}{2mm}
	\item [(1.1)] $\overline{T}$, $\hat{\overline{T}}$ are 
	$(L_{\overline{T},s},L_{\overline{T},a})$,~$(L_{\hat{\overline{T}},s},L_{\hat{\overline{T}},a})$ Lipschitz w.r.t. states and actions.
	
	\item [(1.2)] $\mathcal{A}$ is a convex, closed, bounded (diameter $\diam_{\mathcal{A}}$) set in a $\dim_{\mathcal{A}}$-dimensional space. 
	
	\item [(1.3)]  $\pi(a|s)\sim \mathcal{P}_{\mathcal{A}}[\mathcal{N}(\mu_\pi(s)$,$\Sigma_\pi(s))]$ and 
	$\pi_D(a|s)\sim \mathcal{P}_{\mathcal{A}}[\mathcal{N}(\mu_{\pi_D}(s),\Sigma_{\pi_D}(s))]$.
	
	\item [(1.4)] $\mu_\pi$, $\mu_{\pi_D}$, $\Sigma_{\pi_D}^{1/2}$, $\Sigma_\pi^{1/2}$ are $L_{\pi,\mu}$,~$L_{\pi_D,\mu}$,~$L_{\pi,\Sigma}$,~$L_{\pi_D,\Sigma}$ Lipschitz w.r.t. states.
\end{itemize}

In (1.3), $\mathcal{P}_{\mathcal{A}}$ is the projection to $\mathcal{A}$ and in (1.4)
$\|\Sigma_\pi^{1/2}(s)-\Sigma_\pi^{1/2}(s')\|
\leq L_{\pi,\Sigma}\norm{s-s'}_2$. Assumption 1 is easily satisfied in most continuous control tasks, as explained in \S4 of the Appendix. The harder one, which will be resolved later, is $\gamma \eta_{\pi,\overline{T}}<1$ in Lemma~\ref{lemma:cond-contr}.

\begin{lemma}\label{lemma:cond-contr}
	If Assumption 1 holds, and $\eta_{\pi,\overline{T}}=L_{\overline{T},s}+L_{\overline{T},a}(L_{\pi,\mu}+L_{\pi,\Sigma}  \sqrt{\dim_\mathcal{A}})<1/\gamma$, then $B_{\pi,\overline{T}}$ is a $\gamma\eta_{\pi,\overline{T}}$-contraction w.r.t. 1-Wasserstein distance.
\end{lemma}

To verify there exists a nontrivial system such that 
the condition $\gamma \eta_{\pi,\overline{T}}<1$ 
in Lemma~\ref{lemma:cond-contr} holds under Assumption 1, we consider a continuous-control task. The key term depends on the sample interval $\Delta$. Let
$s=[x,v]^{\top}=[\text{position},~\text{velocity}]^\top$, and $a=\text{acceleration}$. By the laws of motion,
\begin{equation}
s'= \begin{bmatrix} x' \\ v'\end{bmatrix}=\begin{bmatrix}
x+ v \Delta +\frac{1}{2} a\Delta^2 \\ v+ a\Delta  
\end{bmatrix}
=\begin{bmatrix}
I & I\Delta \\ 0 & I
\end{bmatrix} s +\begin{bmatrix}
I\frac{1}{2}\Delta^2\\  I\Delta
\end{bmatrix}a=\overline{T}(s,a)
\label{eq:motion}
\end{equation}
This shows $L_{\overline{T},s}= 1+O(\Delta)$, $L_{\overline{T},a}=O(\Delta)$ and $\eta_{\pi,\overline{T}}=1+O(\Delta)$. Therefore, we conclude that $\gamma \eta_{\pi,\overline{T}} = \gamma + O(\Delta )<1$ for small enough $\Delta$. 

If Lemma~\ref{lemma:cond-contr} holds for $\hat{\overline{T}}$, we can apply Eq.~\eqref{eq:rough}: measure error in $W_1$ distance, apply contraction on $W_1$ to get an asymmetric bound (Lemma~\ref{lemma:sym-bridge}) and then upper bound $W_1$ distance by $\ell_2$ error. This gives the following Theorem for deterministic transitions.

\begin{theorem}\label{thm:err-l2-lip-trans}
	Under Lemma~\ref{lemma:cond-contr}, if $r(s,a)$ is $L_r$-Lipschitz and the $\ell_2$ error is $\epsilon_{\ell_2}$, then
	\begin{align*}	
	|R(\pi_D,\overline{T})-R(\pi_D,\hat{\overline{T}})|
	\leq(1+L_{\pi_D,\mu}+L_{\pi_D,\Sigma}\sqrt{\dim_{\mathcal{A}}})L_r\frac{\gamma \epsilon_{\ell_2}}{(1-\gamma)(1-\gamma\eta_{\pi_D,\hat{\overline{T}}})}.
	\end{align*}
\end{theorem}
The typical MuJoCo reward, $r(s,a)=\text{velocity}(s)-\norm{a}_2^2,$ is Lipschitz if the diameter $\diam_{\mathcal{A}}$ is finite. As MuJoCo also provides the bounds for the action space $\mathcal{A}$, the Lipschitz assumption on $r(s,a)$ is usually satisfied in MuJuCo continuous-control tasks. 

Theorem~\ref{thm:err-l2-lip-trans}, in conjunction with the calculation of $\eta$ in Eq.~\eqref{eq:motion}, indicates that the cumulative model error decreases as the sampling interval $\Delta$ becomes smaller. This is because the cumulative model error decreases in $\eta_{\pi,\hat{\overline{T}}}$ and $\eta_{\pi,\hat{\overline{T}}}=1+O(\Delta)$. Hence, as expected, the sampling period of a continuous-control task has to be small enough in order to train an MBRL system.

Although Theorem~\ref{thm:err-l2-lip-trans} requires Lemma~\ref{lemma:cond-contr}'s strong assumption, branched rollouts allow this assumption to be satisfied since branched rollouts use a much smaller discount factor. Thus, one might expect a benefit from using branched rollouts with deterministic transitions. This is validated in the following Corollary.

\begin{corollary}\label{cor:err-deter-lip-branch}
	Let $r(s,a)$ be $L_r$-Lipschitz and bounded: $0\leq r(s,a)\leq r^{\max}$. Let
	
	~~~~~~~~~~~~~$\epsilon_{\pi_D,\pi}^{\overline{T},\gamma}=\E_{s\sim\rho_{\overline{T},\gamma}^{\pi_D}}D_{TV}(\pi_D(\cdot |s) \,\|\, \pi(\cdot|s))$,~~~ $\epsilon_{\pi_D,\pi}^{\hat{\overline{T}},\beta}=\E_{s\sim\rho_{\hat{\overline{T}},\beta}^{\rho_{\overline{T},\gamma}^{\pi_D},\pi}}D_{TV}(\pi_D(\cdot|s) \, \| \, \pi(\cdot|s)),$
	
	~~~~~~~~~~~~~$\epsilon_{\ell_2,\beta}=\E_{(s,a)\sim\rho_{\overline{T},\beta}^{\rho_{\overline{T},\gamma}^{\pi_D},\pi_D}}\|\overline{T}(s,a)-\hat{\overline{T}}(s,a)\|_2$. Then,
	\begin{align*}
	\Big|R_\gamma(\rho_0,\pi,\overline{T})-\frac{1-\beta}{1-\gamma}R_\beta(\rho_{\overline{T},\gamma}^{\pi_D},\pi,\overline{T})\Big|&\leq r^{\max}\Big(\frac{\epsilon_{\pi_D,\pi}^{\overline{T},\gamma}\gamma}{(1-\gamma)^2} + \frac{\epsilon_{\pi_D,\pi}^{\hat{\overline{T}},\beta}\beta}{(1-\beta)(1-\gamma)} + \frac{\epsilon_{\pi_D,\pi}^{\overline{T},\gamma}+\epsilon_{\pi_D,\pi}^{\hat{\overline{T}},\beta}}{1-\gamma} + \frac{\beta}{\gamma-\beta}\Big)\\
	&+(1+L_{\pi_D,\mu}+L_{\pi_D,\Sigma}\sqrt{\dim_{\mathcal{A}}})L_r\frac{\beta \epsilon_{\ell_2,\beta}}{(1-\gamma)(1-\beta\eta_{\pi_D,\hat{\overline{T}}})}
	\end{align*}
\end{corollary}
\vspace{-3mm}
Corollary~\ref{cor:err-deter-lip-branch} shows an additional benefit of branched rollouts: to ensure $\beta \eta_{\pi_D,\hat{\overline{T}}}<1$, by choosing a small $\beta$ (say 0.9). This suggests that branched rollouts are particularly useful for deterministic transitions. Such a suggestion on branched length (or equivalently, the branched discount factor $\beta$) supports the experimental success of \citet{Janner2019mbpo} and their choice of hyperparameter, as mentioned in the last paragraph of \S~4.2. Also, this result is for deterministic transitions, so this resolves an open issue in \citet{Janner2019mbpo}, as they proved for stochastic transitions but experimented with deterministic transitions.

\subsubsection{Weak Lipschitz Continuity}
\label{sec:main-det-weak}

When Lemma~\ref{lemma:cond-contr} is invalid, there is no Bellman contraction, and we cannot use the bounding principle in Eq.~(\ref{eq:rough}). We provide another way to analyze the error, giving a weaker one-sided bound.

We cannot expect much when $L_{\hat{\overline{T}},s}\gg 1,$ since the rollout diverges when being repeatedly applied to $\hat{\overline{T}},$ with the error growing exponentially w.r.t. rollout length. Hence in this subsection we assume $L_{\hat{\overline{T}},s}\leq 1+(1-\gamma)\iota$ with $\iota<1$; i.e., the Lipschitzness of the transition w.r.t. state is slightly higher than $1$. The longer the expected length $(1-\gamma)^{-1}$, the smoother $\hat{\overline{T}}$ should be. 

The following theorem reveals the impact of asymmetry when there is no Bellman contraction. 

\begin{theorem}
	Let $0\leq r(s,a)\leq r^{\max}$ and 
	$\epsilon_{\ell_2}=\E_{(s,a)\sim \rho_{\overline{T}}^{\pi_D}}
	\|{\overline{T}-\hat{\overline{T}}}\|_2.$
	Assume that:
	\vspace{-2mm}
	\begin{itemize}
		\setlength\topsep{0pt}
		\setlength\itemsep{-1pt} 
		\setlength\parsep{0pt}
		\setlength{\itemindent}{0mm}
		\item [(a)] $\hat{\overline{T}}(s,a),$ $r(s,a),$ $\pi_D(a|s)$ are Lipschitz in $s$ for any $a$ with constants $L_{\hat{\overline{T}},s},$ $L_{r,s},$ 
		$L_{\pi_D,s}.$
		\item [(b)] $L_{\hat{\overline{T}},s}\leq 1+(1-\gamma)\iota$ with $\iota< 1$.
		\item [(c)] The action space is bounded: $\diam_\mathcal{A}<\infty$. Then,
	\end{itemize}
	$$	
	R(\pi_D,\overline{T})-R(\pi_D,\hat{\overline{T}})\leq\frac{1+\gamma}{(1-\gamma)^2}\sqrt{2\epsilon_{\ell_2} r^{\max}L_r}+\frac{1+O(\iota)}{(1-\gamma)^{5/2}}r^{\max}\sqrt{2\epsilon_{\ell_2} L_{\pi_D}\diam_\mathcal{A}}.
	$$
	\label{thm:err-l2-trans}
\end{theorem}

Theorem~\ref{thm:err-l2-trans} is a one-sided bound resulting from the asymmetry of $\epsilon_{\ell_2}=\E_{(s,a)\sim \rho_{\overline{T}}^{\pi_D}}
\|{\overline{T}-\hat{\overline{T}}}\|_2$: $\E$ is taken biasedly on $\rho_{\overline{T}}^{\pi_D}$, so we can only upper bound $R(\pi_D,\overline{T})$ by $R(\pi_D,\hat{\overline{T}})+O((1-\gamma)^{-5/2})$. The resulting MBRL error (Corollary 6, Appendix) only ensures that a policy that works well on $\overline{T}$ also works on $\hat{\overline{T}}$, but not the other way around. This one-sided nature may allow $\hat{\overline{T}}$ to overfit the data. This supports the use of the Ensemble Method \citep{kurutach2018metrpo} to mitigate model bias by training multiple independent models.

Theorem~\ref{thm:err-l2-trans} only indicates the consequence of the $\epsilon_{\ell_2}$ objective's asymmetry. However, this is avoidable. As discussed in Corollary~\ref{cor:err-deter-lip-branch}, branched rollouts provide a Bellman contraction and hence two-sided bounds. 

\section{Conclusion}

Using a Bellman flow contraction w.r.t. distance metrics of probability distributions, we have provided results on the cumulative reward error in  MBRL for both stochastic and deterministic transitions. In particular,  absolutely continuous stochastic transitions and deterministic transitions with strong Lipschitz continuity have Bellman contractions. This result suggests that MBRL is better suited to these situations. The difficulty of dealing with deterministic transitions that do not yield a Bellman contraction arises from the objective function's asymmetry. Finally, we prove that branched rollouts can significantly reduce the error of MBRL and allow a Bellman contraction under deterministic transitions.

\bibliographystyle{humannat}
\bibliography{mybib}

\newpage
\onecolumn

\appendix
\section{Appendix}
\setcounter{lemma}{0}
\setcounter{theorem}{0}
\setcounter{corollary}{0}
\setcounter{assumption}{0}

\subsection{Total Variations, Bellman Contraction and Symmetry Bridge}
\label{appendix:bellman}
\begin{fact}
	Let $m^1$ and $m^2$ be probability measures on $\mathbb{R}^n$ whose singular continuous parts are zero. Decompose $m^1$ and $m^2$ into their absolutely continuous and discrete parts: $m^1=m^1_a+m^1_d$, $m^2=m^2_a+m^2_d$. Then 
	$$D_{TV}(m^1\lvert\rvert m^2)=\frac{1}{2}\big(\norm{m^1_a-m^2_a}_1+\norm{m^1_d-m^2_d}_1\big)\triangleq \frac{1}{2}\norm{m^1-m^2}_1.$$
	\label{fact:tv-l1}
\end{fact}
\begin{proof}
	\cite[Theorem 19.20]{analysis} implies $D_{TV}(m^1\lvert\rvert m^2)=D_{TV}(m^1_a\lvert\rvert m^2_a)+D_{TV}(m^1_d\lvert\rvert m^2_d)$, so Fact~\ref{fact:tv-l1} is proved by combining \cite[Theorem 19.20]{analysis} and that TV distance = half of $\ell_1$ norm for absolutely continuous or discrete measures.
	
	~~~\\
	To avoid using a big hammer, we provide an alternative proof by revising the usual proof of ``TV distance = half of $\ell_1$ norm" with the Lebesgue decomposition: $m=m_a+m_d$. Since $m_a^1$ and $m_a^2$ are absolutely continuous w.r.t. Lebesgue measure, let $d^1$,~$d^2$ be the corresponding probability density functions.
	
	~~~\\
	Let $B=B_a\cup B_d$ where $B_a=\{x\in \text{Supp}(m_a^1)\cup \text{Supp}(m_a^2):~d^1(x)\geq d^2(x)\}$,~$B_d=\{x\in \text{Supp}(m_d^1)\cup \text{Supp}(m_d^2):~m^1_d(x)\geq m^2_d(x)\}$. Since $m_a$ and $m_d$ are mutually singular, we know
	\begin{equation}
	m_a^1(B_d)=m_a^2(B_d)=0=m_d^1(B_a)=m_d^2(B_a)
	\label{eq:singular}
	\end{equation}
	Also, the complement operation implies
	\begin{equation}
	m^2(A^c)-m^1(A^c)=1-m^2(A)-1+m^1(A)=m^1(A)-m^2(A),~~~\text{for~any~measurable~set}~A
	\label{eq:diff-equivalence}
	\end{equation}
	Hence we have an important result
	\begin{equation}
	\begin{split}
	m^1(B)-m^2(B)=&m^1_a(B)-m^2_a(B) + m^1_d(B)-m^2_d(B)\\
	\overset{(\ref{eq:singular})}{=}&m^1_a(B_a)-m^2_a(B_a) + m^1_d(B_d)-m^2_d(B_d)\\
	\overset{(\ref{eq:diff-equivalence})}{=}&\frac{1}{2}\big[m^1_a(B_a)-m^2_a(B_a)+m^2_a(B_a^c)-m^1_a(B_a^c)+m^1_d(B_d)-m^2_d(B_d)+m^2_d(B_d^c)-m^1_d(B_d^c)\big]\\
	=&\frac{1}{2}\Big[\int_{B_a}d^1(x)-d^2(x)dx+\int_{\text{Supp}(m^1_a)\cup\text{Supp}(m^2_a)\backslash B_a}d^2(x)-d^1(x)dx\\
	&+\sum_{x\in B_d}m_d^1(x)-m_d^2(x) +\sum_{x\in \text{Supp}(m_d^1)\cup \text{Supp}(m_d^2)\backslash B_d}m_d^2(x)-m_d^1(x)\Big]\\
	=&\frac{1}{2}\Big[\int_{\text{Supp}(m^1_a)\cup\text{Supp}(m^2_a)}\Big|d^1(x)-d^2(x)\Big|dx +\sum_{x\in \text{Supp}(m_d^1)\cup \text{Supp}(m_d^2)}\Big|m_d^1(x)-m_d^2(x)\Big|\Big]\\
	=&\frac{1}{2}\big(\norm{m^1_a-m^2_a}_1+\norm{m^1_d-m^2_d}_1\big)\triangleq \frac{1}{2}\norm{m^1-m^2}_1.
	\end{split}
	\label{eq:l1-extension}
	\end{equation}
	\begin{enumerate}[label=(\roman*)]
		\item By definition of TV distance, we get
		\begin{align*}
		D_{TV}(m^1\lvert\rvert m^2)\geq |m^1(B)-m^2(B)|=m^1(B)-m^2(B)\overset{(\ref{eq:l1-extension})}{=}\frac{1}{2}\norm{m^1-m^2}_1
		\end{align*}
		\item For any measurable set $A$ in $\mathbb{R}^n$, we know
		$$m^1(A)-m^2(A)=\big[m^1(A\cap B)-m^2(A\cap B)\big]+\big[m^1(A\cap B^c)-m^2(A\cap B^c)\big]$$
		By definition of $B$, the first term is nonnegative while the second term is nonpositive; therefore
		\begin{align*}
		|m^1(A)-m^2(A)|\leq& \max\Big\{m^1(A\cap B)-m^2(A\cap B),~m^2(A\cap B^c)-m^1(A\cap B^c)\Big\}\\
		\leq & \max\Big\{m^1(B)-m^2(B),~m^2( B^c)-m^1( B^c)\Big\}\\
		\overset{(\ref{eq:diff-equivalence})}{=}&m^1(B)-m^2(B)\overset{(\ref{eq:l1-extension})}{=}\frac{1}{2}\norm{m^1-m^2}_1
		\end{align*}
		Taking a supremium over $A$, we arrive at
		$$D_{TV}(m^1\lvert\rvert m^2)\leq \frac{1}{2}\norm{m^1-m^2}_1.$$
	\end{enumerate}
	Combining (i) and (ii), the result follows.
\end{proof}

Due to Fact~\ref{fact:tv-l1}, in the following we will treat TV distance as the half of $\ell_1$ norm. Also, to unify the operations in discrete and continuous parts, we will consider ``generalized" probability density functions where Dirac delta function is included. Thus, Fact~\ref{fact:tv-l1} is rephrased as
$$D_{TV}(m^1\lvert\rvert m^2)= \frac{1}{2}\int |d^1(x)-d^2(x)|dx,$$
where $d^1$, $d^2$ are the generalized density functions of $m^1$ and $m^2$. This allows us to prove Fact~\ref{fact:contra-l1}:
\begin{fact}
	$B_{\pi,T}$ is a $\gamma$-contraction w.r.t. total variation distance.
	\label{fact:contra-l1}
\end{fact}
\begin{proof}
	Let $p_1(s),~p_2(s)$ be the density functions of some state distributions.
	\begin{equation*}
	\begin{split}
	D_{TV}(B_{\pi,T}(p_1) \lvert\rvert B_{\pi,T}(p_2)) &=\frac{1}{2}\int \big\lvert B_{\pi,T}(p_1(s)) - B_{\pi,T}(p_2(s)) \big\rvert ds\\
	&=\frac{1}{2}\int \gamma \Big\lvert \int T(s|s',a')\pi(a'|s') \big(p_1(s') - p_2(s') \big) ds'da'\Big\rvert ds\\
	&\leq \frac{\gamma}{2}\int T(s|s',a')\pi(a'|s')\big\lvert p_1(s')-p_2(s') \big\rvert ds'da'ds\\
	&=\frac{\gamma}{2} \int \big\lvert p_1(s')-p_2(s') \big\rvert ds'=\gamma D_{TV}(p_1\lvert\rvert p_2).
	\end{split}
	\end{equation*}
\end{proof}
The advantages of working on contractions are their convergence and unique fixed-point properties [Theorem 1.1.]\cite{Conrad2014}.

\begin{fact}
	Let $(X,d)$ be a complete metric space and $f:~X\rightarrow X$ be a map such that
	$$d(f(x),~f(x'))\leq cd(x,x')$$
	for some $0\leq c<1$ and all $x,~x'\in X$. Then $f$ has a unique fixed point in $X$. Moreover, for any $x_0\in X$ the sequence of the iterates $x_0$,~$f(x_0)$,~$f(f(x_0))$,... converges to the fixed point of $f$. 
	\label{fact:contra}
\end{fact}

\begin{fact}
	The normalized state occupancy measure $\rho_{T,\gamma}^{\rho_0,\pi}(s)$ is a fixed point of the Bellman flow operator $B_{\pi,T}(\cdot)$.
	\label{fact:fixed-pt}
\end{fact}
\begin{proof}
	\begin{equation*}
	\begin{split}
	\rho_{T,\gamma}^{\rho_0,\pi}(s)=&(1-\gamma)\sum_{i=0}^\infty \gamma^i f_i(s|\rho_0,\pi,T)\\
	=&(1-\gamma)f_{0}(s|\rho_0,\pi,T)+\gamma(1-\gamma)\sum_{i=0}^\infty \gamma^i f_{i+1}(s|\rho_0,\pi,T)\\
	=&(1-\gamma)\rho_0(s)+\gamma(1-\gamma)\sum_{i=0}^\infty \gamma^i \int T(s|s',a')\pi(a'|s')f_i(s'|\rho_0,\pi,T)ds'da'\\
	=&(1-\gamma)\rho_0(s)+\gamma \int T(s|s',a')\pi(a'|s')(1-\gamma)\sum_{i=0}^\infty\gamma^i f_i(s'|\rho_0,\pi,T)ds'da'\\
	=&(1-\gamma)\rho_0(s)+\gamma \int T(s|s',a')\pi(a'|s')\rho_{T,\gamma}^{\rho_0,\pi}(s')ds'da'=B_{\pi,T}(\rho_{T,\gamma}^{\rho_0,\pi}(s)).
	\end{split}
	\end{equation*}
\end{proof}
Together, Fact~\ref{fact:contra-l1} and \ref{fact:contra} imply the Bellman flow operator has a unique fixed point, and according to Fact~\ref{fact:fixed-pt}, the unique fixed point is the state occupancy measure. The contraction and the fixed point properties are particularly useful for proving the symmetry bridge Lemma.

\begin{lemma}[symmetry bridge]
	Let $B_2$ be a Bellman flow operator with fixed-point $\rho_2$. Let $\rho_1$ be another state distribution. If $B_2$ is a $\eta$-contraction w.r.t. some metric $\norm{\cdot}$, then $\norm{\rho_1-\rho_2}\leq \norm{\rho_1-B_2(\rho_1)}/(1-\eta)$.
	\label{lemma:sym-bridge2}
\end{lemma}
\begin{proof}
	\begin{equation*}
	\begin{split}
	\norm{\rho_1-\rho_2}&=\norm{\rho_1-B_2^\infty(\rho_1)}\leq\norm{\rho_1- B_2(\rho_1)}+ \sum\limits_{i=1}^\infty \norm{B_2^i(\rho_1)- B_2^{i+1}(\rho_1)}\\
	&\leq \norm{\rho_1- B_2(\rho_1)}+ \sum\limits_{i=1}^\infty \norm{\rho_1- B_2(\rho_1)}\eta^i=\norm{\rho_1- B_2(\rho_1)}/(1-\eta).
	\end{split}
	\end{equation*}
	The first line uses the fixed-point property and the triangle inequlaity for the distance metric $\norm{\cdot}$. The second line uses the contraction property.
\end{proof}

\subsection{Error of Policies}

\begin{lemma}[Error w.r.t. TV Distance between Occupancy Measures]
	Let $\rho_1(s,a),~\rho_2(s,a)$ be two normalized occupancy measures of rollouts with discount factor $\gamma$. If $0\leq r(s,a)\leq r^{\max}$, then $|R(\rho_1)- R(\rho_2)| \leq D_{TV}(\rho_1\lvert\rvert \rho_2) r^{\max}/(1-\gamma)$.
	where $D_{TV}$ is the total variation distance.
	\label{lemma:err-tv-occu2}
\end{lemma}
\begin{proof}
	\begin{equation*}
	\begin{split}
	R(\rho_1)&=\frac{1}{1-\gamma}\int r(s,a)\rho_1(s,a)dsda\leq\frac{1}{1-\gamma}\int r(s,a)\max\big(\rho_1(s,a),\rho_2(s,a)\big)dsda\\
	&=R(\rho_2)+\frac{1}{1-\gamma}\int r(s,a)\Big(\max\big(\rho_1(s,a),\rho_2(s,a)\big)-\rho_2(s,a)\Big)dsda\\
	&\leq R(\rho_2)+\frac{r^{\max}}{1-\gamma}\int \max\big(\rho_1(s,a),\rho_2(s,a)\big)-\rho_2(s,a)dsda\\
	&=R(\rho_2)+\frac{r^{\max}}{1-\gamma}\frac{1}{2}\norm{\rho_1- \rho_2}_1= R(\rho_2)+\frac{r^{\max}}{1-\gamma}D_{TV}(\rho_1\lvert\rvert \rho_2).
	\end{split}
	\end{equation*}
	Because the TV distance is symmetric, we may interchange the roles of $\rho_1$ and $\rho_2$; thus we conclude that
	$$|R(\rho_1)-R(\rho_2)|\leq D_{TV}(\rho_1\lvert\rvert \rho_2) r^{\max}/(1-\gamma).$$
\end{proof}

\begin{theorem}[Error of Policies]
	If $0\leq r(s,a)\leq r^{\max}$ and the discrepancy in policies is \\$\epsilon_{\pi_D,\pi}^T=\E_{s\sim \rho_T^{\pi_D}}[D_{TV}(\pi_D(\cdot|s)\lvert\rvert \pi(\cdot|s))]$, then $|R(\pi_D,T)-R(\pi,T)|\leq \epsilon_{\pi_D,\pi}^T r^{\max}\Big(\frac{1}{1-\gamma}+\frac{\gamma}{(1-\gamma)^2}\Big).$
	\label{thm:err-tv-policy2}
\end{theorem}
\begin{proof}
	Let $B_{\pi_D,T}$,~$B_{\pi,T}$ be Bellman flow operators whose fixed points are $\rho_T^{\pi_D}(s)$, $\rho_T^{\pi}(s)$, respectively.
	
	According to Lemma~\ref{lemma:err-tv-occu2}, we need to upper bound $D_{TV}(\rho_T^{\pi_D}(s,a)\lvert\lvert \rho_T^{\pi}(s,a))$. Observe that 
	\begin{equation}
	\begin{split}
	D_{TV}(\rho_T^{\pi_D}(s,a)\lvert\lvert \rho_T^{\pi}(s,a))&=\frac{1}{2}\int\Big|\rho_T^{\pi_D}(s,a)-\rho_T^{\pi}(s,a)\Big|dsda=\frac{1}{2}\int\Big|\rho_T^{\pi_D}(s)\pi_D(a|s)-\rho_T^{\pi}(s)\pi(a|s)\Big|dsda\\
	&\leq \frac{1}{2}\int\rho_T^{\pi_D}(s)\Big|\pi_D(a|s)-\pi(a|s)\Big|+\pi(a|s)\Big|\rho_T^{\pi_D}(s)- \rho_T^{\pi}(s)\Big|dsda\\
	&=\epsilon_{\pi_D,\pi}^T + D_{TV}(\rho_T^{\pi_D}(s)\lvert\lvert \rho_T^{\pi}(s))
	\label{eq:tv-occu-pi}
	\end{split}
	\end{equation}
	As for the rest, by the properties of the Bellman flow operators, we have
	\begin{equation}
	\begin{split}
	D_{TV}(\rho_T^{\pi_D}(s)\lvert\lvert \rho_T^{\pi}(s))&\leq \frac{1}{1-\gamma}D_{TV}(\rho_T^{\pi_D}(s)\lvert\lvert B_{\pi,T}(\rho_T^{\pi_D}(s)))\\
	&=\frac{1}{1-\gamma}D_{TV}(B_{\pi_D,T}(\rho_T^{\pi_D}(s))\lvert\lvert B_{\pi,T}(\rho_T^{\pi_D}(s)))\\
	&\leq \frac{\gamma}{2(1-\gamma)}\int T(s|s',a')\Big|\pi_D(a'|s')-\pi(a'|s')\Big|\rho_T^{\pi_D}(s')ds'da'ds\\
	&=\frac{\gamma}{1-\gamma}\epsilon_{\pi_D,\pi}^T,
	\end{split}
	\label{eq:tv-state-pi}
	\end{equation}
	where the top two lines follows from the symmetry bridge property (Lemma~\ref{lemma:sym-bridge2}) and the fixed-point property.
	Combining Eq.~(\ref{eq:tv-occu-pi}) and (\ref{eq:tv-state-pi}), we know $D_{TV}(\rho_T^{\pi_D}(s,a)\lvert\lvert \rho_T^{\pi}(s,a))\leq \epsilon_{\pi_D,\pi}^T(1+\frac{\gamma}{1-\gamma})$; therefore by Lemma~\ref{lemma:err-tv-occu2},
	$$|R(\pi_D,T)-R(\pi,T)|\leq \epsilon_{\pi_D,\pi}^T r^{\max}\Big(\frac{1}{1-\gamma}+\frac{\gamma}{(1-\gamma)^2}\Big)$$
\end{proof}
\begin{corollary}[Error of Behavior Cloning]
	Let $\pi_D$ and $\pi$ be the expert policy and the agent policy. If $0\leq r(s,a)\leq r^{\max}$ and $\E_{s\sim\rho_{T}^{\pi_D}}D_{KL}(\pi_D(\cdot|s)\lvert\rvert\pi(\cdot|s))\leq \epsilon_{BC}$, then $|R(\pi_D,T)-R(\pi,T)|\leq \sqrt{\epsilon_{BC}/2} r^{\max}\Big(\frac{1}{1-\gamma}+\frac{\gamma}{(1-\gamma)^2}\Big)$.
\end{corollary}
\begin{proof}
	The result is immediate from Theorem~\ref{thm:err-tv-policy2} and the Pinsker's Inequality.
\end{proof}
\begin{corollary}[Error of GAIL]
	Let $\pi_D$ and $\pi$ be the expert policy and the agent policy. If $0\leq r(s,a)\leq r^{\max}$ and $D_{JS}(\rho_{T}^{\pi_D}\lvert\rvert \rho_{T}^{\pi})\leq \epsilon_{GAIL}$. Then $|R(\pi_D,T)-R(\pi,T)|\leq \sqrt{2\epsilon_{GAIL}}r^{\max}/(1-\gamma)$
\end{corollary}
\begin{proof}
	By definition of the JSD, for any distributions $P,~Q$ and their average $M=(P+Q)/2$ we know
	\begin{equation*}
	D_{JS}(P\lvert\rvert Q) = \frac{1}{2}\big[D_{KL}(P\lvert\rvert M)+ D_{KL}(Q\lvert\rvert M)\big]\geq D_{TV}(P\lvert\rvert M)^2+ D_{TV}(Q\lvert\rvert M)^2\geq \frac{1}{2}D_{TV}(P\lvert\rvert Q)^2,
	\end{equation*}
	where the first inequality follows from Pinsker's Inequality, and the second inequality holds because that $D_{TV}(P\lvert\rvert M)+D_{TV}(Q\lvert\rvert M)\geq D_{TV}(P\lvert\rvert Q)$ by triangle inequality and that $2a^2+2b^2\geq c^2$ if $a+b\geq c\geq 0$.
	
	Thus, we know $D_{TV}(\rho_{T}^{\pi_D}\lvert\rvert \rho_{T}^{\pi})\leq \sqrt{2\epsilon_{GAIL}}$. Applying Lemma~\ref{lemma:err-tv-occu2} completes the proof.
\end{proof}
\subsection{MBRL with Absolutely Continuous Stochastic Transitions}
\begin{theorem}[Error of Absolutely Continuous Stochastic Transitions]
	Let $\pi_D$, $T$ and $\hat{T}$ be the sampling policy, the real and the learned transitions. If $0\leq r(s,a)\leq r^{\max}$ and the error in one-step total variation distance is $\epsilon_{T,\hat{T}}^{\pi_D}=\E_{(s,a)\sim \rho_T^{\pi_D}}[D_{TV}(T(\cdot|s,a)\lvert\rvert \hat{T}(\cdot|s,a))]$, then $|R(\pi_D,T)-R(\pi_D,\hat{T})|\leq \epsilon_{T,\hat{T}}^{\pi_D}  r^{\max}\gamma(1-\gamma)^{-2}$.
	\label{thm:err-tv-trans2}
\end{theorem}
\begin{proof}
	If there is a upper bound for $D_{TV}(\rho_T^{\pi_D}(s,a)\lvert\rvert\rho_{\hat{T}}^{\pi_D}(s,a))$, by Lemma~\ref{lemma:err-tv-occu2}, we are done. Also, observe that 
	\begin{equation*}
	D_{TV}(\rho_T^{\pi_D}(s,a)\lvert\rvert\rho_{\hat{T}}^{\pi_D}(s,a))=\frac{1}{2}\int\pi_D(a|s) \big\lvert\rho_T^{\pi_D}(s)-\rho_{\hat{T}}^{\pi_D}(s)\big\rvert dsda=D_{TV}(\rho_T^{\pi_D}(s)\lvert\rvert\rho_{\hat{T}}^{\pi_D}(s)),
	\end{equation*}
	so $D_{TV}(\rho_T^{\pi_D}(s)\lvert\rvert\rho_{\hat{T}}^{\pi_D}(s))$ is of interest. Employing the properties of Bellman flow operator, we have
	\begin{equation*}
	\begin{split}
	D_{TV}(\rho_T^{\pi_D}(s)\lvert\rvert\rho_{\hat{T}}^{\pi_D}(s))\leq&\frac{1}{1-\gamma} D_{TV}(\rho_T^{\pi_D}(s)\lvert\rvert B_{\pi_D,\hat{T}}(\rho_T^{\pi_D}(s)))\\
	=&\frac{1}{1-\gamma}D_{TV}(B_{\pi_D,T}(\rho_T^{\pi_D}(s))\lvert\rvert B_{\pi_D,\hat{T}}(\rho_T^{\pi_D}(s)))\\ 
	=&\frac{1}{2(1-\gamma)}\int \Big\lvert \gamma\int \big(T(s|s',a')-\hat{T}(s|s',a')\big)\pi_D(a'|s') \rho_T^{\pi_D}(s')ds'da'\Big\rvert ds\\
	\leq& \frac{\gamma}{2(1-\gamma)}\int \Big\lvert \big(T(s|s',a')-\hat{T}(s|s',a')\big) \Big\rvert\rho_T^{\pi_D}(s',a')dsds'da'\\
	=&\frac{\gamma}{1-\gamma} \E_{(s,a)\sim\rho_T^{\pi_D} } D_{TV}(T(\cdot|s,a)\lvert\rvert \hat{T}(\cdot|s,a))= \frac{\gamma}{1-\gamma}\epsilon_{T,\hat{T}}^{\pi_D},
	\end{split}
	\end{equation*}
	where the top two lines follows from the symmetry bridge property (Lemma~\ref{lemma:sym-bridge2}) and the fixed-point property.
	Finally, from Lemma~\ref{lemma:err-tv-occu2}, we conclude that
	$$|R(\pi_D,T)-R(\pi_D,\hat{T})|\leq \epsilon_{T,\hat{T}}^{\pi_D} r^{\max}\frac{\gamma}{(1-\gamma)^2}$$
\end{proof}

\begin{corollary}[Error of MBRL with Absolutely Continuous Stochastic Transition]
	Let $\pi_D$, $\pi$, $T$ and $\hat{T}$ be the sampling policy, the agent policy, the real transition and the learned transition. If $0\leq r(s,a)\leq r^{\max}$ and the discrepancies are $\epsilon_{T,\hat{T}}^{\pi_D}=\E_{(s,a)\sim\rho_T^{\pi_D}}D_{TV}(T(\cdot|s,a)\lvert\rvert \hat{T}(\cdot|s,a))$ and $\epsilon_{\pi_D,\pi}^{T,\gamma}=\E_{s\sim\rho_T^{\pi_D}}D_{TV}(\pi_D(\cdot|s)\lvert\rvert \pi(\cdot|s))$, then $|R(\pi,T)-R(\pi,\hat{T})|\leq (\epsilon_{T,\hat{T}}^{\pi_D}+\epsilon_{\pi_D,\pi}^{T,\gamma}+\epsilon_{\pi_D,\pi}^{\hat{T},\gamma})r^{\max}\gamma/(1-\gamma)^2 +(\epsilon_{\pi_D,\pi}^{T,\gamma}+\epsilon_{\pi_D,\pi}^{\hat{T},\gamma})r^{\max}/(1-\gamma)$.
\end{corollary}
\begin{proof}
	Observe that $|R(\pi,T)-R(\pi,\hat{T})|\leq |R(\pi,T)-R(\pi_D,T)|+|R(\pi_D,T)-R(\pi_D,\hat{T})|+|R(\pi_D,\hat{T})-R(\pi,\hat{T})|$. Combining Theorem~\ref{thm:err-tv-trans2} and~\ref{thm:err-tv-policy2}, the result follows.
\end{proof}

\begin{lemma}
	Let $\gamma> \beta$ be discount factors of long and short rollouts. Let $\pi_D$ and $T$ be the sampling policy and the real transition, then $D_{TV}(\rho_{T,\gamma}^{\pi_D}\lvert\rvert \rho_{T,\beta}^{\rho_{T,\gamma}^{\pi_D},\pi_D})\leq (1-\gamma)\beta/(\gamma-\beta)$.
	\label{lemma:short_occu_bound2}
\end{lemma}
\begin{proof}
	Since $\rho_{T,\gamma}^{\pi_D}$ is generated by the triple $(\rho_0,\pi_D,T)$ with discount factor $\gamma$ while $\rho_{T,\beta}^{\rho_{T,\gamma}^{\pi_D},\pi_D}$ is generated by $(\rho_{T,\gamma}^{\pi_D},\pi_D,T)$ with discount factor $\beta$. By definition of the occupancy measure we have
	\begin{equation*}
	\begin{split}
	&\rho_{T,\gamma}^{\pi_D}(s,a)=\sum\limits_{i=0}^\infty (1-\gamma)\gamma^i f_i(s,a).\\
	&\rho_{T,\beta}^{\rho_{T,\gamma}^{\pi_D},\pi_D}(s,a)=\sum\limits_{i=0}^\infty \sum\limits_{j=0}^i (1-\gamma)\gamma^{i-j}(1-\beta)\beta^j f_i(s,a),
	\end{split}
	\end{equation*}
	where $f_i(s,a)$ is the density of $(s,a)$ at time $i$ if generated by the triple $(\rho_0,\pi_D,T)$. Then,
	\begin{equation*}
	\begin{split}
	D_{TV}(\rho_{T,\gamma}^{\pi_D}\lvert\rvert \rho_{T,\beta}^{\rho_{T,\gamma}^{\pi_D},\pi_D})&\leq \frac{1}{2}\sum\limits_{i=0}^\infty\Big|(1-\gamma)\gamma^i - \sum\limits_{j=0}^i(1-\gamma)\gamma^{i-j}(1-\beta)\beta^j \Big|=\frac{1}{2}\sum\limits_{i=0}^\infty (1-\gamma)\gamma^i\Big| 1-\sum\limits_{j=0}^i (1-\beta)\Big(\frac{\beta}{\gamma}\Big)^j \Big|\\
	&=\frac{1}{2}\sum\limits_{i=0}^\infty (1-\gamma)\gamma^i\frac{1}{\gamma-\beta}\Big|-\beta(1-\gamma) + \Big(\frac{\beta}{\gamma}\Big)^{i+1}(1-\beta)\gamma  \Big|\\
	&\overset{(*)}{=} \frac{(1-\gamma)\beta}{\gamma-\beta}\sum\limits_{i=0}^{M-1} -(1-\gamma)\gamma^i + (1-\beta)\beta^i=\frac{(1-\gamma)\beta}{\gamma-\beta}(\gamma^{M}-\beta^{M})\\
	&\leq \frac{(1-\gamma)\beta}{\gamma-\beta}.
	\end{split}
	\end{equation*}
	where $(*)$ comes from that $-\beta(1-\gamma)+(\frac{\beta}{\gamma})^i(1-\beta)\gamma$ is a strictly decreasing function in $i$. Since $\gamma>\beta$, its sign flips from $+$ to $-$ at some index; say $M$. Finally, the sum of the absolute value are the same between $\sum_{i=0}^{M-1}$ and $\sum_{i=M}^\infty$ because the total probability is conservative, and the difference on one side is the same as that on the other.
\end{proof}

\begin{corollary}[Error of MBRL with A. C. Stochastic Transition and Branched Rollouts]
	Let $\gamma> \beta$ be discount factors of long and short rollouts. Let $\pi_D$, $\pi$, $T$ and $\hat{T}$ be sampling policy, agent policy, real transition and learned transition. If $0\leq r(s,a)\leq r^{\max}$ and the discrepancies are $\epsilon_{\pi_D,\pi}^{T,\gamma}=\E_{s\sim\rho_{T,\gamma}^{\pi_D}}D_{TV}(\pi_D(\cdot|s)\lvert\rvert \pi(\cdot|s))$,\\ $\epsilon_{\pi_D,\pi}^{\hat{T},\beta}=\E_{s\sim\rho_{\hat{T},\beta}^{\rho_{T,\gamma}^{\pi_D},\pi}}D_{TV}(\pi_D(\cdot|s)\lvert\rvert \pi(\cdot|s))$,~~~and~,$\epsilon_{T,\hat{T}}^{\pi_D,\beta}=\E_{(s,a)\sim\rho_{T,\beta}^{\rho_{T,\gamma}^{\pi_D},\pi_D}}D_{TV}(T(\cdot|s,a)\lvert\rvert \hat{T}(\cdot|s,a))$, then $$\Big|R_\gamma(\rho_0,\pi,T)-\frac{1-\beta}{1-\gamma}R_\beta(\rho_{T,\gamma}^{\pi_D},\pi,\hat{T})\Big|\leq r^{\max}\Big(\frac{\epsilon_{\pi_D,\pi}^{T,\gamma}\gamma}{(1-\gamma)^2} + \frac{(\epsilon_{T,\hat{T}}^{\pi_D,\beta} + \epsilon_{\pi_D,\pi}^{\hat{T},\beta})\beta}{(1-\beta)(1-\gamma)} + \frac{\epsilon_{\pi_D,\pi}^{T,\gamma}+\epsilon_{\pi_D,\pi}^{\hat{T},\beta}}{1-\gamma} + \frac{\beta}{\gamma-\beta}\Big)$$
	\label{cor:err-branch2}
\end{corollary}
\begin{proof}
	Expand with the triangle inequality:
	\begin{equation*}
	\begin{split}
	&\Big|R_\gamma(\rho_0,\pi,T)-\frac{1-\beta}{1-\gamma}R_\beta(\rho_{T,\gamma}^{\pi_D},\pi,\hat{T})\Big|\\
	\leq&\Big|R_\gamma(\rho_0,\pi,T)-R_\gamma(\rho_0,\pi_D,T)\Big|+\Big|R_\gamma(\rho_0,\pi_D,T)-\frac{1-\beta}{1-\gamma}R_\beta(\rho_{T,\gamma}^{\pi_D},\pi_D,T)\Big|+\\
	&\frac{1-\beta}{1-\gamma}\Big|R_\beta(\rho_{T,\gamma}^{\pi_D},\pi_D,T)-R_\beta(\rho_{T,\gamma}^{\pi_D},\pi_D,\hat{T})\Big|+\frac{1-\beta}{1-\gamma}\Big|R_\beta(\rho_{T,\gamma}^{\pi_D},\pi_D,\hat{T})-R_\beta(\rho_{T,\gamma}^{\pi_D},\pi,\hat{T})\Big|
	\end{split}
	\end{equation*}
	By Theorem~\ref{thm:err-tv-policy2}, the first term $\leq \epsilon_{\pi_D,\pi}^{T,\gamma} r^{\max}\Big(\frac{1}{1-\gamma}+\frac{\gamma}{(1-\gamma)^2}\Big)$.\\
	The second term is a short extension of Lemma~\ref{lemma:err-tv-occu2} and Lemma~\ref{lemma:short_occu_bound2}:
	\begin{equation*}
	\begin{split}
	R_\gamma(\rho_0,\pi_D,T)&=\frac{1}{1-\gamma}\int r(s,a)\rho_{T,\gamma}^{\pi_D}(s,a)dsda\leq\frac{1}{1-\gamma}\int r(s,a)\max\big(\rho_{T,\gamma}^{\pi_D}(s,a),\rho_{T,\beta}^{\rho_{T,\gamma}^{\pi_D},\pi_D}(s,a)\big)dsda\\
	&=\frac{1-\beta}{1-\gamma}R_\beta(\rho_{T,\gamma}^{\pi_D},\pi_D,T)+\frac{1}{1-\gamma}\int r(s,a)\Big(\max\big(\rho_T^{\pi}(s,a),\rho_{T,\beta}^{\rho_{T,\gamma}^{\pi_D},\pi_D}(s,a)\big)-\rho_{T,\beta}^{\rho_{T,\gamma}^{\pi_D},\pi_D}(s,a)\Big)dsda\\
	&\leq \frac{1-\beta}{1-\gamma}R_\beta(\rho_{T,\gamma}^{\pi_D},\pi_D,T)+\frac{r^{\max}}{1-\gamma}\int \Big(\max\big(\rho_T^{\pi}(s,a),\rho_{T,\beta}^{\rho_{T,\gamma}^{\pi_D},\pi_D}(s,a)\big)-\rho_{T,\beta}^{\rho_{T,\gamma}^{\pi_D},\pi_D}(s,a)\Big)dsda\\
	&\leq\frac{1-\beta}{1-\gamma}R_\beta(\rho_{T,\gamma}^{\pi_D},\pi_D,T)+\frac{r^{\max}}{1-\gamma}D_{TV}(\rho_{T,\gamma}^{\pi_D}\lvert\rvert \rho_{T,\beta}^{\rho_{T,\gamma}^{\pi_D},\pi_D})
	\end{split}
	\end{equation*}
	By the symmetry of the total variation distance and Lemma~\ref{lemma:short_occu_bound2}, we obtain
	$$\Big|R_\gamma(\rho_0,\pi_D,T)-\frac{1-\beta}{1-\gamma}R_\beta(\rho_{T,\gamma}^{\pi_D},\pi_D,T)\Big|\leq \frac{r^{\max}}{1-\gamma}D_{TV}(\rho_{T,\gamma}^{\pi_D}\lvert\rvert \rho_{T,\beta}^{\rho_{T,\gamma}^{\pi_D},\pi_D})\leq r^{\max}\frac{\beta}{\gamma-\beta}.$$
	By Theorem~\ref{thm:err-tv-trans2}, the third term $\leq \epsilon_{T,\hat{T}}^{\pi_D,\beta}r^{\max}\frac{\beta}{(1-\beta)(1-\gamma)}$.\\
	By Theorem~\ref{thm:err-tv-policy2}, the fourth term $\leq \epsilon_{\pi_D,\pi}^{\hat{T},\beta} r^{\max}\Big(\frac{1}{1-\gamma}+\frac{\beta}{(1-\beta)(1-\gamma)}\Big)$.
\end{proof}

\subsection{MBRL with Deterministic Transition and Strong Lipschitz Continuity}
\begin{assumption}
	~~~\\
	\begin{itemize}[topsep=3pt, itemsep=1pt, itemindent=-2mm, leftmargin=12mm]
		\item [(1.1)] $\overline{T}$, $\hat{\overline{T}}$ are 
		$(L_{\overline{T},s},L_{\overline{T},a})$,~$(L_{\hat{\overline{T}},s},L_{\hat{\overline{T}},a})$ Lipschitz w.r.t. states and actions.
		
		\item [(1.2)] $\mathcal{A}$ is a convex, closed, bounded (diameter $\diam_{\mathcal{A}}$) set in a $\dim_{\mathcal{A}}$-dimensional space. 
		
		\item [(1.3)]  $\pi(a|s)\sim \mathcal{P}_{\mathcal{A}}[\mathcal{N}(\mu_\pi(s),\Sigma_\pi(s))]$ and 
		$\pi_D(a|s)\sim \mathcal{P}_{\mathcal{A}}[\mathcal{N}(\mu_{\pi_D}(s),\Sigma_{\pi_D}(s))]$
		
		\item [(1.4)] $\mu_\pi$, $\mu_{\pi_D}$, $\Sigma_{\pi_D}^{1/2}$, and $\Sigma_\pi^{1/2}$ are $L_{\pi,\mu}$,~$L_{\pi_D,\mu}$,~$L_{\pi,\Sigma}$,~$L_{\pi_D,\Sigma}$ Lipschitz w.r.t. states.
		\label{assump:lip2}
	\end{itemize}
\end{assumption}
The validation of Assumption~\ref{assump:lip2} is below.
\begin{itemize}
	\item [1.1] \textbf{The real and learned transitions are Lipschitz w.r.t. states and actions.} For the real transition especially in continuous control, the Lipschitzness follows from the laws of motion, as computed in Eq. (10) in the paper. For the learned transition, the Lipschitzness can be made by spectral normalization \citep{miyato2018spectral} or gradient penalty \citep{Gulrajani2017wgan-gp}, which are some notable approaches to ensure the Lipschitzness of the discriminator in Wasserstein GAN \citep{arjovsky2017wasserstein}.
	\item [1.2] \textbf{The action space is convex, closed and bounded in a finite dimensional linear space.} This is a standard assumption in continuous-control and is usually satisfied (or made satisfied) in practice \citep{fujita18act-bounded}. The boundness assumption, if not naturally satisfied, is addressed in 1.3.  
	\item [1.3] \textbf{The policy follows truncated Gaussian, by projecting the Gaussian r.v. onto the action space.} According to \citep{fujita18act-bounded}, this is a common practice in RL experiment. The Gaussain assumption is made by training some NNs for the mean and variance of the policies. As for the projection of action to a bounded convex set, it is perfectly fine in RL experiment and is largely used in most MuJoCo experiments as MuJoCo also provides the bounds for the action space. It is also a good practice since it helps stabilize the training.
	\item [1.4] \textbf{The mean and covariance of the policy are Lipschitz w.r.t. state.} As again noted in 1.1, the Lipschitzness can be realized by spectral normalization or gradient penalty. Since the mean and covariance of the policy are represented by some NN, this assumption can be easily made in practice.
\end{itemize}
\begin{lemma}[Conditional Contraction]
	Under assumption~\ref{assump:lip2}, if $\eta_{\pi,\overline{T}}=L_{\overline{T},s}+L_{\overline{T},a}(L_{\pi,\mu}+L_{\pi,\Sigma}\sqrt{\dim_\mathcal{A}})<1/\gamma$, where $\gamma$ is the discount factor of $B_{\pi,\overline{T}}$, then $B_{\pi,\overline{T}}$ is a $\gamma\eta_{\pi,\overline{T}}$-contraction w.r.t. 1-Wasserstein distance.
	\label{lemma:cond-contr2}
\end{lemma}
\begin{proof}
	Recall $B_{\pi,\overline{T}}(\rho(s))=(1-\gamma)\rho_0(s)+\gamma\int \delta(s-\overline{T}(s',a'))\pi(a'|s')\rho(s')ds'da'.$ 
	Let $\rho_1(s)$,~$\rho_2(s)$ be some distributions over states. We have
	\begin{equation*}
	\begin{split}
	&\wass{B_{\pi,\overline{T}}(\rho_1)}{B_{\pi,\overline{T}}(\rho_2)}\\
	\overset{(a)}{\leq}& \gamma \inf_{J(s_1,a_1,s_2,a_2)\sim \Pi(\rho_1(s)\pi(a|s), \rho_2(s)\pi(a|s))} \E_J \norm{\overline{T}(s_1,a_1)-\overline{T}(s_2,a_2)}_2\\
	=& \gamma \inf_{J(s_1,a_1,s_2,a_2)\sim \Pi(\rho_1(s)\pi(a|s), \rho_2(s)\pi(a|s))} \E_J \norm{\overline{T}(s_1,a_1)-\overline{T}(s_1,a_2)+\overline{T}(s_1,a_2)-\overline{T}(s_2,a_2)}_2\\
	\leq&\gamma \inf_{J(s_1,a_1,s_2,a_2)\sim \Pi(\rho_1(s)\pi(a|s), \rho_2(s)\pi(a|s))} \E_J L_{\overline{T},a}\norm{a_1-a_2}_2+L_{\overline{T},s}\norm{s_1-s_2}_2\\
	\overset{(b)}{=}&\gamma\inf_{J(s_1,\xi_1,s_2,\xi_2)\sim \Pi(\rho_1,\mathcal{N}, \rho_2,\mathcal{N})} \E_J L_{\overline{T},a}\norm{\mathcal{P}_{\mathcal{A}}[\mu_\pi(s_1)+\Sigma_\pi^{1/2}(s_1)\xi_1]-\mathcal{P}_{\mathcal{A}}[\mu_\pi(s_2)-\Sigma_\pi^{1/2}(s_2)\xi_2]}_2+L_{\overline{T},s}\norm{s_1-s_2}_2\\
	\overset{(c)}{\leq}&\gamma\inf_{J(s_1,\xi_1,s_2,\xi_2)\sim \Pi(\rho_1,\mathcal{N}, \rho_2,\mathcal{N})} \E_J L_{\overline{T},a}\norm{\mu_\pi(s_1)+\Sigma_\pi^{1/2}(s_1)\xi_1-\mu_\pi(s_2)-\Sigma_\pi^{1/2}(s_2)\xi_2}_2+L_{\overline{T},s}\norm{s_1-s_2}_2\\
	\leq& \gamma \inf_{J(s_1,s_2)\sim \Pi(\rho_1, \rho_2)}\Big( \E_J (L_{\overline{T},s}+L_{\overline{T},a}L_{\pi,\mu})\norm{s_1-s_2}_2+\inf_{K(\xi_1,\xi_2)\sim\Pi(\mathcal{N},\mathcal{N})}\E_{K} L_{\overline{T},a}\norm{\Sigma_\pi^{1/2}(s_1)\xi_1-\Sigma_\pi^{1/2}(s_2)\xi_2}_2\Big)\\
	\overset{(d)}{\leq}&\gamma \inf_{J(s_1,s_2)\sim \Pi(\rho_1, \rho_2)}\Big( \E_J (L_{\overline{T},s}+L_{\overline{T},a}L_{\pi,\mu})\norm{s_1-s_2}_2+\E_{\xi_1}L_{\overline{T},a}\norm{\Sigma_\pi^{1/2}(s_1)-\Sigma_\pi^{1/2}(s_2)}_{op}\norm{\xi_1}_2\Big)\\
	\overset{(e)}{\leq}& \gamma\inf_{J(s_1,s_2)\sim \Pi(\rho_1, \rho_2)}\E_J (L_{\overline{T},s}+L_{\overline{T},a}(L_{\pi,\mu}+L_{\pi,\Sigma}\sqrt{\dim_\mathcal{A}}))\norm{s_1-s_2}_2\\
	=& \gamma(L_{\overline{T},s}+L_{\overline{T},a}(L_{\pi,\mu}+L_{\pi,\Sigma}\sqrt{\dim_\mathcal{A}}))\wass{\rho_1}{\rho_2}=\gamma\eta_{\pi,\overline{T}}\wass{\rho_1}{\rho_2},
	\end{split}
	\end{equation*}
	where $\underset{J(s_1,s_2)\sim \Pi(\rho_1,\rho_2)}{\inf}$ takes a infimum over all joint distributions $J(s_1,s_2)$ whose marginals are $\rho_1$ and $\rho_2$.
	(a) selects a joint distribution over $B_{\pi,\overline{T}}(\rho_1)$ and $B_{\pi,\overline{T}}(\rho_2)$ that share the same randomness of $(1-\gamma)\rho_0$, which establishes a upper bound and allows us to cancel $(1-\gamma)\rho_0$. (b) uses the Gaussian assumption of the policy, with $\xi_1,~\xi_2$ being standard normal vectors. (c) uses the non-expansiveness property of projection onto a closed convex set.(d) selects $\xi_1=\xi_2$ and uses the property of operator norm. (e) uses the Lipschitz assumption of $\Sigma_\pi^{1/2}(s)$ and that $\norm{\xi_1}\leq\sqrt{\dim_\mathcal{A}}$ by Jensen inequality.
\end{proof}

\begin{lemma}[Error w.r.t. W1 Distance between Occupancy Measures]
	Let $\rho_1(s,a),~\rho_2(s,a)$ be two normalized occupancy measures of rollouts with discount factor $\gamma$. If the reward is $L_r$-Lipschitz, then $|R(\rho_1)- R(\rho_2)| \leq \wass{\rho_1}{\rho_2} L_r/(1-\gamma)$.
	\label{lemma:err-w1-occu2}
\end{lemma}
\begin{proof}
	The cumulative reward is bounded by
	\begin{equation*}
	\begin{split}
	R(\rho_1)&=\frac{1}{1-\gamma}\int r(s,a)\rho_1(s,a)dsda=R(\rho_2)+\frac{1}{1-\gamma}\int r(s,a)\big(\rho_1(s,a)-\rho_2(s,a)\big)dsda\\
	&=R(\rho_2)+\frac{L_r}{1-\gamma}\int \frac{r(s,a)}{L_r}\big(\rho_1(s,a)-\rho_2(s,a)\big)dsda\\
	&\leq R(\rho_2)+\frac{L_r}{1-\gamma}\underset{\norm{f}_{\text{Lip}}\leq 1}{\sup}\int f(s,a)\big(\rho_1(s,a)-\rho_2(s,a)\big)dsda\\
	&=R(\rho_2)+\frac{L_r}{1-\gamma}\underset{\norm{f}_{\text{Lip}}\leq 1}{\sup}\E_{(s,a)\sim \rho_1}[f(s,a)]-\E_{(s,a)\sim \rho_2}[f(s,a)]\\
	&=R(\rho_2)+\frac{L_r}{1-\gamma}\wass{\rho_1}{\rho_2}.
	\end{split}
	\end{equation*}
	The third line holds because $r(s,a)/L_r$ is 1-Lipschitz and the last line follows from Kantorovich-Rubinstein duality \cite{Villani2008opt_trans}. Since $W_1$ distance is symmetric, the same conclusion holds if interchanging $\rho_1$ and $\rho_1$; thus
	$$|R(\rho_1)-R(\rho_2)|\leq \wass{\rho_1}{\rho_2} L_r/(1-\gamma).$$
\end{proof}

\begin{theorem}[Error of Deterministic Transitions with Strong Lipschitzness]
	Under Lemma~\ref{lemma:cond-contr2}, let $\overline{T}$, $\hat{\overline{T}}$, $r$,~$\pi_D$ be deterministic real transition, deterministic learned transtion, reward and sampling policy. If $r(s,a)$ is $L_r$-Lipschitz and the $\ell_2$ error is $\epsilon_{\ell_2}$, then $|R(\pi_D,\overline{T})-R(\pi_D,\hat{\overline{T}})|\leq(1+L_{\pi_D,\mu}+L_{\pi_D,\Sigma}\sqrt{\dim_{\mathcal{A}}})L_r\frac{\gamma \epsilon_{\ell_2}}{(1-\gamma)(1-\gamma\eta_{\pi_D,\hat{\overline{T}}})}$.
	\label{thm:err-l2-lip-trans2}
\end{theorem}
\begin{proof}
	Observe that the Wasserstein distance over the joint can be upper bounded by that over the marginal.
	\begin{equation}
	\begin{split}
	\wass{\rho_{\overline{T}}^{\pi_D}(s,a)}{\rho_{\hat{\overline{T}}}^{\pi_D}(s,a)}=&\underset{J(s_1,a_1,s_2,a_2)\in\Pi(\rho_{\overline{T}}^{\pi_D}(s,a),\rho_{\hat{\overline{T}}}^{\pi_D}(s,a))}{\inf}\E_J\norm{(s_1-s_2,a_1-a_2)}_2\\
	\leq& \underset{J(s_1,a_1,s_2,a_2)\in\Pi(\rho_{\overline{T}}^{\pi_D}(s,a),\rho_{\hat{\overline{T}}}^{\pi_D}(s,a))}{\inf}\E_J\norm{s_1-s_2}_2+\norm{a_1-a_2}_2\\
	\overset{(*)}{\leq}&(1+L_{\pi_D,\mu}+L_{\pi_D,\Sigma}\sqrt{\dim_{\mathcal{A}}})\underset{J(s_1,s_2)\in\Pi(\rho_{\overline{T}}^{\pi_D}(s),\rho_{\hat{\overline{T}}}^{\pi_D}(s))}{\inf}\E_J\norm{s_1-s_2}_2\\
	=&(1+L_{\pi_D,\mu}+L_{\pi_D,\Sigma}\sqrt{\dim_{\mathcal{A}}})\wass{\rho_{\overline{T}}^{\pi_D}(s)}{\rho_{\hat{\overline{T}}}^{\pi_D}(s)},
	\end{split}
	\label{eq:w1-joint-marg}
	\end{equation}
	where (*) follows from the same analysis in Lemma~\ref{lemma:cond-contr2}. Also, the Wasserstein distance over the marginal is upper bounded by the $\ell_2$ error:
	\begin{equation}
	\begin{split}
	\wass{\rho_{\overline{T}}^{\pi_D}(s)}{\rho_{\hat{\overline{T}}}^{\pi_D}(s)}\leq& \frac{1}{1-\gamma\eta_{\pi_D,\hat{\overline{T}}}}\wass{\rho_{\overline{T}}^{\pi_D}(s)}{B_{\hat{\overline{T}}}^{\pi_D}(\rho_{\overline{T}}^{\pi_D}(s))}
	= \frac{1}{1-\gamma\eta_{\pi_D,\hat{\overline{T}}}}\wass{B_{\overline{T}}^{\pi_D}(\rho_{\overline{T}}^{\pi_D}(s))}{B_{\hat{\overline{T}}}^{\pi_D}(\rho_{\overline{T}}^{\pi_D}(s))}\\
	\leq &\frac{\gamma}{1-\gamma\eta_{\pi_D,\hat{\overline{T}}}}\inf_{J(s_1,a_1,s_2,a_2)\sim \Pi(\rho_{\overline{T}}^{\pi_D}(s)\pi_D(a|s), \rho_{\overline{T}}^{\pi_D}(s)\pi_D(a|s))}\E_J\norm{\overline{T}(s_1,a_1)-\hat{\overline{T}}(s_2,a_2)}_2\\
	\leq& \frac{\gamma}{1-\gamma\eta_{\pi_D,\hat{\overline{T}}}} \E_{(s,a)\sim\rho_{\overline{T}}^{\pi_D}(s)\pi_D(a|s)} \norm{\overline{T}(s,a)-\hat{\overline{T}}(s,a)}_2= \frac{\gamma}{1-\gamma\eta_{\pi_D,\hat{\overline{T}}}}\epsilon_{\ell_2}.
	\end{split}
	\label{eq:w1-l2}
	\end{equation}
	The first line follows from conditional contraction (Lemma~\ref{lemma:cond-contr2}), symmetry bridge (Lemma~\ref{lemma:sym-bridge2}) and fixed-point property. The second line uses the fact that $B_{\overline{T}}^{\pi_D}$ and $B_{\hat{\overline{T}}}^{\pi_D}$ have $1-\gamma$ fraction in common, so we can create a joint distribution to cancel it. The third line builds a upper bound by choosing $(s_1,a_1)=(s_2,a_2)\sim\rho_{\overline{T}}^{\pi_D}(s)\pi_D(a|s)$. Combining Eq.~(\ref{eq:w1-joint-marg}),~(\ref{eq:w1-l2}) and Lemma~\ref{lemma:err-w1-occu2}, we conclude that 
	$$|R(\pi_D,\overline{T})-R(\pi_D,\hat{\overline{T}})|\leq (1+L_{\pi_D,\mu}+L_{\pi_D,\Sigma}\sqrt{\dim_{\mathcal{A}}})L_r\frac{\gamma \epsilon_{\ell_2}}{(1-\gamma)(1-\gamma\eta_{\pi_D,\hat{\overline{T}}})}.$$
\end{proof}

\begin{corollary}[Error of MBRL with Deterministic Transition, Strong Lipschitzness and Branched Rollouts]
	Let $\gamma> \beta$ be discount factors of long and short rollouts. Let $\pi_D$, $\pi$, $\overline{T}$ and $\hat{\overline{T}}$ be sampling policy, agent policy, real deterministic transition and deterministic learned transition. Under assumption~\ref{assump:lip2}, suppose the reward is both bounded $0\leq r(s,a)\leq r^{\max}$ and $L_r$-Lipschitz. Let $\epsilon_{\pi_D,\pi}^{\overline{T},\gamma}=\E_{s\sim\rho_{\overline{T},\gamma}^{\pi_D}}D_{TV}(\pi_D(\cdot|s)\lvert\rvert \pi(\cdot|s))$,\\ $\epsilon_{\pi_D,\pi}^{\hat{\overline{T}},\beta}=\E_{s\sim\rho_{\hat{\overline{T}},\beta}^{\rho_{\overline{T},\gamma}^{\pi_D},\pi}}D_{TV}(\pi_D(\cdot|s)\lvert\rvert \pi(\cdot|s))$ and $\epsilon_{\ell_2,\beta}=\E_{(s,a)\sim\rho_{\overline{T},\beta}^{\rho_{\overline{T},\gamma}^{\pi_D},\pi_D}}\norm{\overline{T}(s,a)-\hat{\overline{T}}(s,a)}_2$. Then,
	\begin{align*}
	\Big|R_\gamma(\rho_0,\pi,\overline{T})-\frac{1-\beta}{1-\gamma}R_\beta(\rho_{\overline{T},\gamma}^{\pi_D},\pi,\overline{T})\Big|&\leq r^{\max}\Big(\frac{\epsilon_{\pi_D,\pi}^{\overline{T},\gamma}\gamma}{(1-\gamma)^2} + \frac{\epsilon_{\pi_D,\pi}^{\hat{\overline{T}},\beta}\beta}{(1-\beta)(1-\gamma)} + \frac{\epsilon_{\pi_D,\pi}^{\overline{T},\gamma}+\epsilon_{\pi_D,\pi}^{\hat{\overline{T}},\beta}}{1-\gamma} + \frac{\beta}{\gamma-\beta}\Big)\\
	&+(1+L_{\pi_D,\mu}+L_{\pi_D,\Sigma}\sqrt{\dim_{\mathcal{A}}})L_r\frac{\beta \epsilon_{\ell_2,\beta}}{(1-\gamma)(1-\beta\eta_{\pi_D,\hat{\overline{T}}})}
	\end{align*}
	
\end{corollary}
\begin{proof}
	Modifying the proof of Corollary~\ref{cor:err-branch2} with Theorem~\ref{thm:err-l2-lip-trans2}, the result follows.
\end{proof}

\subsection{MBRL with Deterministic Transition and Weak Lipschitz Continuity}
\begin{theorem}[One-sided Error of Deterministic Transitions]
	Let $\overline{T}$, $\hat{\overline{T}}$, $r$,~$\pi_D$ be deterministic real transition, deterministic learned transtion, reward and sampling policy. Suppose $0\leq r(s,a)\leq r^{\max}$. $\hat{\overline{T}}(s,a)$,~$r(s,a)$~and\\$\pi_D(a|s)$ are Lipschitz in $s$ for any $a$ with constants $(L_{\hat{\overline{T}}},L_r,L_{\pi_D})$. Assume that $L_{\hat{\overline{T}}}\leq 1+(1-\gamma)\iota$ with $\iota< 1$ and that the action space is bounded: $\diam_\mathcal{A}<\infty$. If the training loss in $\ell_2$ error is $\epsilon_{\ell_2}$, then
	$$R(\pi_D,\overline{T})-R(\pi_D,\hat{\overline{T}})\leq\frac{1+\gamma}{(1-\gamma)^2}\sqrt{2\epsilon_{\ell_2} r^{\max}L_r}+\frac{1+O(\iota)}{(1-\gamma)^{5/2}}r^{\max}\sqrt{2\epsilon_{\ell_2} L_{\pi_D}\diam_\mathcal{A}}.$$
	\label{thm:err-l2-trans2}
\end{theorem}
\begin{proof}
	Recall the $\ell_2$ error is
	$\E_{(s,a)\sim \rho_{\overline{T}}^{\pi_D}}\Big[\norm{\overline{T}(s,a)-\hat{\overline{T}}(s,a)}_2\Big]=\epsilon_{\ell_2}$. By Markov's Inequality, for any $\delta>0$,
	
	\begin{equation}
	\mathbb{P}_{(s,a)\sim \rho_{\overline{T}}^{\pi_D}}\Big( \norm{\overline{T}(s,a)-\hat{\overline{T}}(s,a)}_2 <\delta \Big)>1-\frac{\epsilon_{\ell_2}}{\delta}
	\label{markov}
	\end{equation}
	
	Eq.~(\ref{markov}) means for a length $H\sim\text{Geometric}(1-\gamma)$ rollout $\{s_t,a_t\}_{t=1}^H$ generated by $(\rho_0,\pi_D,\overline{T})$,\\ $\norm{\overline{T}(s_t,a_t)-\hat{\overline{T}}(s_t,a_t)}_2<\delta$ with probability greater than $1-\frac{\epsilon_{\ell_2}}{\delta}$.
	
	Following this idea, \emph{we say a rollout is consistent to $\hat{\overline{T}}$, if for each $t$, $\norm{s_{t+1}-\hat{\overline{T}}(s_t,a_t)}_2 <\delta$}; in other words, a rollout is consistent to $\hat{\overline{T}}$ if for each time step, the state transition is similar to what $\hat{\overline{T}}$ does. Let $P_{\overline{T}}$ be the probability measure induced on the rollout following the real transition $\overline{T}$. The cumulative reward is bounded by
	
	\begin{equation}
	\begin{split}
	R(\pi_D,\overline{T}) &= \int_{\text{traj}} R(\text{traj})dP_{\overline{T}}= \int_{\text{traj~consistent}} R(\text{traj})dP_{\overline{T}} + \int_{\text{traj~inconsistent}} R(\text{traj})dP_{\overline{T}}\\
	&\leq \int_{\text{traj~consistent}} R(\text{traj})dP_{\overline{T}} + \frac{\epsilon_{\ell_2}}{\delta}\E[H^2]r^{\max}.
	\end{split}
	\label{eq:main-term}
	\end{equation}
	The inequality holds because for a rollout generated by $\overline{T}$ with length $H$, the probability that it is inconsistent to $\hat{\overline{T}}$ is at most $\frac{\epsilon_{\ell_2}}{\delta}H$ by Eq.~(\ref{markov}) and the union bound over $\{s_t,a_t\}_{t=1}^H$. Also, the maximum reward of such rollout is $Hr^{\max}$.
	
	Now, we'd like to change from $P_{\overline{T}}$ to $P_{\hat{\overline{T}}}$ with the Lipschitz assumptions above. It suffices to reset the states $\{s_i\}_{i=1}^H$ so that the transition obeys $\hat{\overline{T}}$. Suppose the new states are
	\begin{equation}
	s_1' = s_1,~~~~~~s_i' = \hat{\overline{T}}(s_{i-1}',a_{i-1}),~\forall~i\geq 2.
	\label{s'}
	\end{equation}
	By the Lipschitzness of $\hat{\overline{T}}$, triangle inequlaity and $\hat{\overline{T}}$-consistency, the distance between $s_i$ and $s_i'$ obeys
	\begin{equation*}
	\begin{split}
	&\norm{s_1-s_1'}_2=0\\
	&\norm{s_i-s_i'}_2\leq \norm{s_i-\hat{\overline{T}}(s_{i-1},a_{i-1})}_2+\norm{\hat{\overline{T}}(s_{i-1},a_{i-1})-\hat{\overline{T}}(s_{i-1}',a_{i-1})}_2\leq \delta + L_{\hat{\overline{T}}}\norm{s_{i-1}-s_{i-1}'}_2,~\forall i\geq 2.
	\end{split}
	\end{equation*}
	That is,
	\begin{equation}
	\norm{s_i-s_i'}_2\leq \delta \sum\limits_{j=0}^{i-2}L_{\hat{\overline{T}}}^j= \delta \frac{L_{\hat{\overline{T}}}^{i-1}-1}{L_{\hat{\overline{T}}}-1} ,~~~\forall i\geq 2.
	\label{s-bound}
	\end{equation}
	
	The difference of cumulative reward between $\text{traj}=\{s_i,a_i\}_{i=1}^H$ and $\text{traj}'=\{s_i',a_i\}_{i=1}^H$ satisfies
	\begin{equation}
	\begin{split}
	R(\text{traj})&=\sum\limits_{i=1}^H r(s_i,a_i)\leq r(s_1',a_1)+ \sum\limits_{i=2}^H r(s_i',a_i)+ L_r\norm{s_i-s_i'}_2\overset{(\ref{s-bound})}{\leq} R(\text{traj}')+ \delta L_r\sum\limits_{i=2}^H \frac{L_{\hat{\overline{T}}}^{i-1}-1}{L_{\hat{\overline{T}}}-1}\\
	&\overset{(\ref{eq:exp-approx})}{\leq} R(\text{traj}')+\delta L_r(H^2/2 + (\E H)^2O(\iota)) ,
	\end{split}
	\label{eq:cum-change}
	\end{equation}
	where (\ref{eq:exp-approx}) results from imposing $L_{\hat{\overline{T}}}=1+\iota(1-\gamma)=1+\frac{\iota}{\E H}$ into the exponential:
	\begin{equation}
	\begin{split}
	\sum\limits_{i=2}^H \frac{L_{\hat{\overline{T}}}^{i-1}-1}{L_{\hat{\overline{T}}}-1}=&\frac{1}{L_{\hat{\overline{T}}}-1}\Big(\frac{L_{\hat{\overline{T}}}^{H}-L_{\hat{\overline{T}}}}{L_{\hat{\overline{T}}}-1} - H+1\Big)=\frac{(1+\frac{\iota}{\E H})^H-\iota\frac{H}{\E H}-1}{\frac{\iota^2}{ (\E H)^2}}\leq \frac{e^{\iota\frac{H}{\E H}}-\iota\frac{H}{\E H}-1}{\frac{\iota^2}{ (\E H)^2}}\\
	=&\frac{(\iota \frac{H}{\E H})^2/2 + O(\iota^3)}{\frac{\iota^2}{ (\E H)^2}}=\frac{H^2}{2}+(\E H)^2O(\iota)
	\end{split}
	\label{eq:exp-approx}
	\end{equation}
	
	Because the transitions are deterministic, $\{s_i'\}_{i=1}^H$ are constant given $s_1,~a_1,...,a_H$, which means the randomness depends on $s_1,~a_1,...,a_H$ (with $\{s_i'\}_{i=1}^H$ being the conditions of $\pi_D$), and the density satisfies
	\begin{equation}
	\begin{split}
	P_{\hat{\overline{T}}}(\text{traj}') =& \rho_0(s_1')\pi_D(a_1|s_1')\prod_{i=2}^H\pi_D(a_i|s_i')\geq \rho_0(s_1)\pi_D(a_1|s_1)\prod_{i=2}^H\big(\pi_D(a_i|s_i)-L_{\pi_D} \norm{s_i-s_i'}_2\big)\\
	\overset{(\ref{s-bound})}{\geq}& \rho_0(s_1)\pi_D(a_1|s_1)\prod_{i=2}^H\Big(\pi_D(a_i|s_i)+\delta L_{\pi_D} \frac{L_{\hat{\overline{T}}}^{i-1}-1}{L_{\hat{\overline{T}}}-1}\Big)
	\geq P_{\overline{T}}(\text{traj})\Big(1-\sum\limits_{i=2}^H \frac{\delta L_{\pi_D}}{\pi_D(a_i|s_i)}\frac{L_{\hat{\overline{T}}}^{i-1}-1}{L_{\hat{\overline{T}}}-1}\Big)
	\end{split}
	\label{eq:den-change}
	\end{equation}
	Then, conditioning on the length of rollout being $H$, the integral term in Eq.~(\ref{eq:main-term}) is bounded by
	\begin{equation}
	\begin{split}
	&\int_{\text{traj~consistent}|H} R(\text{traj})dP_{\overline{T}}= \int_{s_1,a_1,...,a_H~\text{consis.}} R(\text{traj})P_{\overline{T}}(\text{traj})ds_1da_1...da_H \\
	\overset{(\ref{eq:den-change})}{\leq}& \int_{s_1,a_1,...,a_H~\text{consis.}} R(\text{traj})\Big(P_{\hat{\overline{T}}}(\text{traj}')+P_{\overline{T}}(\text{traj})\sum\limits_{i=2}^H \frac{\delta L_{\pi_D}}{\pi_D(a_i|s_i)}\frac{L_{\hat{\overline{T}}}^{i-1}-1}{L_{\hat{\overline{T}}}-1}\Big)ds_1da_1...da_H\\
	\leq& \int_{s_1,a_1,...,a_H~\text{consis.}} R(\text{traj})P_{\hat{\overline{T}}}(\text{traj}')+\int_{s_1,a_1,...,a_H}R(\text{traj})P_{\overline{T}}(\text{traj})\sum\limits_{i=2}^H \frac{\delta L_{\pi_D}}{\pi_D(a_i|s_i)}\frac{L_{\hat{\overline{T}}}^{i-1}-1}{L_{\hat{\overline{T}}}-1}\\
	\overset{(\ref{eq:cum-change})}{\leq} &  \int_{s_1,a_1,...,a_H~\text{consis.}} \big(R(\text{traj}')+\delta L_r(H^2/2 + (\E H)^2O(\iota))\big)P_{\hat{\overline{T}}}(\text{traj}')ds_1da_1...da_H+\\
	&\int_{s_1,a_1,...,a_H}Hr^{\max}P_{\overline{T}}(\text{traj})\sum\limits_{i=2}^H \frac{\delta L_{\pi_D}}{\pi_D(a_i|s_i)}\frac{L_{\hat{\overline{T}}}^{i-1}-1}{L_{\hat{\overline{T}}}-1} ds_1da_1...da_H\\
	\overset{(\ref{eq:exp-approx})}{\leq} & \int_{s_1,a_1,...,a_H} \big(R(\text{traj}')+\delta L_r(H^2/2 + (\E H)^2O(\iota))\big)P_{\hat{\overline{T}}}(\text{traj}')+
	Hr^{\max}\delta L_{\pi_D}\diam_\mathcal{A}(H^2/2 + (\E H)^2O(\iota))\\
	\leq & R(\pi_D,\hat{\overline{T}})+\delta L_r(H^2/2 + (\E H)^2O(\iota))+\delta L_{\pi_D}r^{\max} \diam_\mathcal{A}(H^3/2 + H(\E H)^2O(\iota))
	\end{split}
	\label{eq:consist-term}
	\end{equation}
	Combining Eq.~(\ref{eq:main-term})~(\ref{eq:consist-term}), by choosing $$\delta=\sqrt{\frac{2\epsilon_{\ell_2} r^{\max}\E[H^2]}{L_r\E[H^2]+\E [H]^2O(\iota)+L_{\pi_D}r^{\max}\diam_\mathcal{A}\big(\E[H^3]+\E[H]^3O(\iota)\big)  }},$$ we are able to minimize:
	$$\frac{\epsilon_{\ell_2}}{\delta}\E[H^2]r^{\max}+\delta L_r(\E[H^2]/2 + (\E H)^2O(\iota))+\delta L_{\pi_D}r^{\max} \diam_\mathcal{A}(\E[H^3]/2 + (\E H)^3O(\iota)),$$
	yielding
	\begin{equation*}
	\begin{split}
	&R(\pi_D,\overline{T})-R(\pi_D,\hat{\overline{T}})\\
	\leq& \E[H^2]\sqrt{\Big(2\epsilon_{\ell_2} r^{\max}\Big)\Big(L_r+\E[H]^2O(\iota)/E[H^2]+L_{\pi_D}r^{\max}\diam_\mathcal{A}\big(\E[H^3]/\E[H^2]+\E[H]^3O(\iota)/\E[H^2]\big)  \Big)}\\
	\overset{(a)}{=}&\E[H^2]\sqrt{2\epsilon_{\ell_2} r^{\max}L_r+2\epsilon_{\ell_2} L_{\pi_D}(r^{\max})^2\diam_\mathcal{A}\big(\E[H^3]/\E[H^2]+\E[H]^3O(\iota)/\E[H^2]\big)}\\
	\overset{(b)}{\leq}&\E[H^2]\sqrt{2\epsilon_{\ell_2} r^{\max}L_r}+\E[H^2]r^{\max}\sqrt{2\epsilon_{\ell_2} L_{\pi_D}\diam_\mathcal{A}\big(\E[H^3]/\E[H^2]+\E[H]^3O(\iota)/\E[H^2]\big)}\\
	=&\E[H^2]\sqrt{2\epsilon_{\ell_2} r^{\max}L_r}+r^{\max}\sqrt{2\epsilon_{\ell_2} L_{\pi_D}\diam_\mathcal{A}}\sqrt{\E[H^2]\big(\E[H^3]+\E[H]^3O(\iota)\big)}\\
	\overset{(c)}{=}&\frac{1+\gamma}{(1-\gamma)^2}\sqrt{2\epsilon_{\ell_2} r^{\max}L_r}+\frac{\sqrt{1+5\gamma+5\gamma^2+\gamma^3+(1+\gamma)O(\iota)}}{(1-\gamma)^{5/2}}r^{\max}\sqrt{2\epsilon_{\ell_2} L_{\pi_D}\diam_\mathcal{A}}\\
	\overset{(d)}{\leq}&\frac{1+\gamma}{(1-\gamma)^2}\sqrt{2\epsilon_{\ell_2} r^{\max}L_r}+\frac{1+O(\iota)}{(1-\gamma)^{5/2}}r^{\max}\sqrt{2\epsilon_{\ell_2} L_{\pi_D}\diam_\mathcal{A}}.
	\end{split}
	\end{equation*}
	(a) merge the two $O(\iota)$ terms together. (b) uses $\sqrt{x+y}\leq \sqrt{x}+\sqrt{y}~\text{for}~x,~y\geq 0$. (c) applies the identities $\E[H^2]=\frac{1+\gamma}{(1-\gamma)^2}$,~$\E[H^3]=\frac{1+4\gamma+\gamma^2}{(1-\gamma)^3}$. (d) uses $\sqrt{1+x}\leq 1+x/2.$
\end{proof}

\begin{corollary}[One-sided of MBRL with Deterministic Transition and Branched Rollouts]
	Let $\gamma> \beta$ be discount factors of long and short rollouts. Let $\pi_D$, $\pi$, $\overline{T}$ and $\hat{\overline{T}}$ be sampling policy, agent policy, real deterministic transition and deterministic learned transition. Under the assumptions of Theorem~\ref{thm:err-l2-trans2}, let $\epsilon_{\pi_D,\pi}^{\overline{T},\gamma}=\E_{s\sim\rho_{\overline{T},\gamma}^{\pi_D}}D_{TV}(\pi_D(\cdot|s)\lvert\rvert \pi(\cdot|s))$,~~$\epsilon_{\pi_D,\pi}^{\hat{\overline{T}},\beta}=\E_{s\sim\rho_{\hat{\overline{T}},\beta}^{\rho_{\overline{T},\gamma}^{\pi_D},\pi}}D_{TV}(\pi_D(\cdot|s)\lvert\rvert \pi(\cdot|s))$\\
	and $\epsilon_{\ell_2,\beta}=\E_{(s,a)\sim \rho_{\overline{T},\beta}^{\rho_{\overline{T},\gamma}^{\pi_D},\pi_D}}\norm{\overline{T}(s,a)-\hat{\overline{T}}(s,a)}_2$. Then~~~ $R_\gamma(\rho_0,\pi,\overline{T})-\frac{1-\beta}{1-\gamma}R_\beta(\rho_{\overline{T},\gamma}^{\pi_D},\pi,\overline{T})\leq r^{\max}\Big(\frac{\epsilon_{\pi_D,\pi}^{\overline{T},\gamma}\gamma}{(1-\gamma)^2} + \frac{\epsilon_{\pi_D,\pi}^{\hat{\overline{T}},\beta}\beta}{(1-\beta)(1-\gamma)} + \frac{\epsilon_{\pi_D,\pi}^{\overline{T},\gamma}+\epsilon_{\pi_D,\pi}^{\hat{\overline{T}},\beta}}{1-\gamma} + \frac{\beta}{\gamma-\beta}\Big)+
	\frac{1+\beta}{(1-\beta)(1-\gamma)}\sqrt{2\epsilon_{\ell_2,\beta} r^{\max}L_r}+\frac{1+O(\iota)}{(1-\beta)^{3/2}(1-\gamma)}r^{\max}\sqrt{2\epsilon_{\ell_2,\beta} L_{\pi_D}\diam_\mathcal{A}}.$
\end{corollary}
\begin{proof}
	Plugging in Theorem~\ref{thm:err-l2-trans2} with $L_{\hat{\overline{T}}}\leq 1+(1-\beta)\iota$ to the proof of Corollary~\ref{cor:err-branch2}, the result follows.
\end{proof}

\end{document}